\documentclass{article}

\usepackage{arxiv}

\usepackage[utf8]{inputenc}
\usepackage[T1]{fontenc}
\usepackage{hyperref}
\usepackage{url}
\usepackage{booktabs}
\usepackage{amsmath,amssymb,amsfonts}
\usepackage{bm}
\usepackage{mathtools}
\usepackage{mathrsfs}
\usepackage{amsthm}
\usepackage{nicefrac}
\usepackage{microtype}
\usepackage{graphicx}
\usepackage{grffile}           % allows spaces in filenames
\usepackage{xcolor}
\usepackage{multirow}
\usepackage{array}
\usepackage{listings}
\usepackage{caption}
\usepackage{siunitx}
\usepackage{algorithm}
\usepackage{algorithmicx}
\usepackage{algpseudocode}
\usepackage{placeins}          % \FloatBarrier
\usepackage[numbers,sort&compress]{natbib}

\definecolor{darkgreen}{rgb}{0.0,0.5,0.0}

\newtheoremstyle{thmstyleone}{8pt}{8pt}{\itshape}{}{\bfseries}{.}{0.5em}{}
\newtheoremstyle{thmstyletwo}{8pt}{8pt}{\normalfont}{}{\itshape}{.}{0.5em}{}
\newtheoremstyle{thmstylethree}{8pt}{8pt}{\normalfont}{}{\bfseries}{.}{0.5em}{}
\theoremstyle{thmstyleone}
\newtheorem{theorem}{Theorem}
\newtheorem{proposition}[theorem]{Proposition}
\theoremstyle{thmstyletwo}

\newtheorem{remark}{Remark}
\theoremstyle{thmstylethree}

\newtheorem{corollary}[theorem]{Corollary}

\newcommand{\R}{\mathbb{R}}
\newcommand{\Eucl}{\mathbb{E}^2}
\newcommand{\Sph}{\mathbb{S}^2}
\newcommand{\Hyp}{\mathbb{H}^2}
\newcommand{\mM}{\mathcal{M}}
\newcommand{\mA}{\mathcal{A}}
\newcommand{\mL}{\mathcal{L}}

\newcommand{\bP}{\mathbf{P}}
\newcommand{\bQ}{\mathbf{Q}}
\newcommand{\bG}{\mathbf{G}}
\newcommand{\bx}{\mathbf{x}}

\newcommand{\balpha}{\bm{\alpha}}
\newcommand{\norm}[1]{\left\lVert#1\right\rVert}
\newcommand{\abs}[1]{\left\lvert#1\right\rvert}
\newcommand{\inner}[2]{\left\langle#1,\,#2\right\rangle}
\newcommand{\EPS}{\varepsilon}
\DeclareMathOperator{\atantwo}{atan2}
\DeclareMathOperator{\arctanh}{arctanh}
\newcommand{\kcode}{\kappa}
\newcommand{\Lip}{\mathrm{Lip}}

\title{A Neural Tension Operator for Curve Subdivision across Constant Curvature Geometries}

\author{
  Hassan Ugail \\
  Centre for Visual Computing and Intelligent Systems \\
  University of Bradford \\
  United Kingdom \\
  \And
  Newton Howard \\
  School of Individualized Study \\
  Rochester Institute of Technology \\
  United States \\
}

\begin{document}
\maketitle

\begin{abstract}
Interpolatory subdivision schemes generate smooth curves from piecewise-linear control polygons by repeatedly inserting new vertices. Classical schemes rely on a single global tension parameter and typically require separate formulations in Euclidean, spherical, and hyperbolic geometries. We introduce a shared learned tension predictor that replaces the global parameter with per-edge insertion angles predicted by a single 140K-parameter network. The network takes local intrinsic features and a trainable geometry embedding as input, and the predicted angles drive geometry-specific insertion operators across all three spaces without architectural modification. A constrained sigmoid output head enforces a structural safety bound, guaranteeing that every inserted vertex lies within a valid angular range for any finite weight configuration. Three theoretical results accompany the method: a structural guarantee of tangent-safe insertions; a heuristic motivation for per-edge adaptivity; and a conditional convergence certificate for continuously differentiable limit curves, subject to an explicit Lipschitz constraint verified post hoc. On 240 held-out validation curves, the learned predictor occupies a distinct position on the fidelity--smoothness Pareto frontier, achieving markedly lower bending energy and angular roughness than all fixed-tension and manifold-lift baselines. Riemannian manifold lifts retain a pointwise-fidelity advantage, which this study quantifies directly. On the out-of-distribution ISS orbital ground-track example, bending energy falls by 41\% and angular roughness by 68\% with only a modest increase in Hausdorff distance, suggesting that the predictor generalises beyond its synthetic training distribution.
\end{abstract}


%%=============================================================================
\FloatBarrier
\section{Introduction}
\label{sec:intro}

Subdivision schemes have been central to curve and surface design since the 1970s, generating smooth geometric objects from piecewise-linear control structures through repeated local refinement \cite{Chaikin1974,Dyn1987,Cavaretta1991}. Among these, the four-point interpolatory scheme of Dyn, Gregory and Levin \cite{Dyn1987} is the most widely studied. Each step inserts a new vertex between consecutive control-polygon vertices via a weighted combination of four neighbours governed by a single global scalar, the \emph{tension parameter}~$\mu$. Setting $\mu = 0$ recovers the classical four-point rule; $\mu = -0.25$ yields the six-point variant achieving $C^2$ limit curves, as characterised by Weissman \cite{Weissman1990}.

Despite their simplicity, classical schemes face two structural limitations. First, $\mu$ is \emph{spatially invariant}: a gently curving arc and a sharp corner demand different tension, yet the scheme applies a single value everywhere. Second, separate formulations are required for different geometries. Wallner and Dyn \cite{Wallner2005} extended the four-point rule to Riemannian manifolds, and specialised variants exist for the sphere \cite{Xie2006,Sabin2004} and the hyperbolic plane \cite{Ahanchaou2026}. Supporting all three curvature regimes thus requires duplicated implementations.

Recent progress in geometric deep learning \cite{Bronstein2021} and neural operators \cite{Li2021,Lu2021} suggests a natural alternative: learning a mapping from local geometric descriptors to a per-edge insertion angle. Such a mapping can replace three independent classical schemes while adapting locally in ways no fixed parameter achieves.

This paper pursues that idea. We reformulate the four-point and six-point schemes on $\Eucl$, $\Sph$ and $\Hyp$ within a unified angle-based parameterisation and introduce a shared residual network with 140{,}505 parameters that predicts per-edge insertion angles driving geometry-specific insertion operators. The prediction architecture and geometry embedding are shared across all three spaces; the geodesic insertion rules and training data remain geometry-specific. Three theoretical results accompany the method: a structural $G^1$ safety bound holding for any finite weights, a heuristic argument for the necessity of per-edge adaptivity, and a conditional $C^1$ convergence guarantee contingent on a Lipschitz constraint verified post hoc.

The three geometries arise naturally in visual computing. Curves on $\Sph$ appear in satellite-track visualisation and panoramic video \cite{Cohen2018Spherical}. Curves in $\Hyp$ underlie Poincar\'e disk graph-layout systems \cite{Nickel2017,Chami2019}. Curves in $\Eucl$ remain the foundation of interactive CAD, where higher-order PDE methods have long provided principled fairness criteria for smooth geometric design \cite{kubiesa2004,monterde2004biharmonic,monterde2006general}. More broadly, replacing manually tuned geometric parameters with learned predictors has demonstrated value across a range of visual computing tasks \cite{ugail2019deepface}, a pattern the present work extends to the subdivision setting.

The contributions of this paper are as follows.

\begin{enumerate}
\item \textbf{Shared learned tension predictor.}
A 140{,}505-parameter residual network, conditioned on local intrinsic features and a trainable geometry embedding, predicts per-edge insertion angles driving geometry-specific insertion operators on $\Eucl$, $\Sph$ and $\Hyp$. A sigmoid-constrained output head guarantees that all inserted vertices lie within a geometrically valid angular range for any finite weights.

\item \textbf{Supporting theoretical analysis.}
We provide a structural $G^1$ safety bound, a heuristic motivation for the insufficiency of any fixed $\mu$ across mixed-curvature polygons, and a conditional $C^1$ convergence argument applying when the trained network satisfies an explicit Lipschitz constraint.

\item \textbf{Empirical evaluation with Pareto-framed results.}
Five baselines, including Riemannian manifold lifts and a validation-set oracle, provide a comprehensive comparison. Fixed-tension methods are uniformly dominated. Manifold lifts achieve higher pointwise fidelity, whereas the learned predictor achieves substantially lower bending energy and $G^1$ roughness. Out-of-distribution generalisation is demonstrated on the ISS orbital ground track, where bending energy and angular variation are reduced significantly with only a modest increase in Hausdorff distance.

\item \textbf{Mechanistic analysis.}
Inverting the effective tension reveals that the network operates far outside the classical range $[-0.5, 0]$ and adopts geometry-specific strategies.
\end{enumerate}

Sections~\ref{sec:related}--\ref{sec:training} cover related work, mathematical foundations, theoretical analysis, architecture and training. Sections~\ref{sec:experiments}--\ref{sec:analysis} report experiments, the ISS application and mechanistic analysis. Sections~\ref{sec:discussion}--\ref{sec:conclusion} discuss and conclude.

\section{Related Work}
\label{sec:related}

Cavaretta et al.,\ \cite{Cavaretta1991} established the foundations of stationary subdivision, and the four-point rule of Dyn, Gregory and Levin \cite{Dyn1987} remains the most widely studied interpolatory scheme. Its smoothness has been analysed by Weissman \cite{Weissman1990} and Hassan and Dodgson \cite{Hassan2002}; comprehensive textbook treatments appear in Farin \cite{Farin2002} and Prautzsch et al.,\ \cite{Prautzsch2002}. Extensions to non-Euclidean settings include the proximity-based Riemannian framework of Wallner and Dyn \cite{Wallner2005}, which underpins our convergence analysis, and the strengthened spherical bounds of H\"uning and Wallner \cite{Huning2022}. Most directly relevant is the angle-based framework of Ahanchaou et al.,\ \cite{Ahanchaou2026}, which provides the unified Euclidean, spherical and hyperbolic formulations we use as classical baselines.

PointNet \cite{Qi2017}, Dynamic Graph CNN \cite{Wang2019}, and MeshCNN \cite{Hanocka2019} established paradigms for per-point and per-edge learning on three-dimensional geometry. Within subdivision, Liu et al.,\ \cite{Liu2020NS} trained a network to predict new vertices for Euclidean Loop subdivision via per-patch weight-sharing, resembling our local design but addressing neither non-Euclidean settings nor analytical guarantees. Neural progressive meshes \cite{Chen2023NPM,Liu2023INDI} extend this direction to surfaces, though they remain restricted to Euclidean domains.

Representation learning on non-Euclidean spaces has grown substantially. Nickel and Kiela \cite{Nickel2017} showed that Poincar\'e embeddings outperform Euclidean ones for hierarchical data, motivating a range of hyperbolic architectures \cite{Chami2019,vanSpengler2023,DesaiHyperbolic2023}. Spherical CNNs \cite{Cohen2018Spherical} provide analogous tools for data on the sphere. The Fourier Neural Operator \cite{Li2021} and DeepONet \cite{Lu2021} popularised learning functional mappings rather than fixed parametric formulae; our predictor shares this philosophy but is a simple local edge-wise multilayer perceptron, far narrower in scope than those frameworks.

%%=============================================================================
\FloatBarrier
\section{Mathematical Foundations}
\label{sec:foundations}

\FloatBarrier
\subsection{Notation and Setting}
\label{subsec:notation}

Let $\mM$ denote one of three model spaces: the Euclidean plane $\Eucl$, the unit two-sphere $\Sph \subset \R^3$, or the hyperbolic plane $\Hyp$ represented by the Poincar\'e open disk,
\[
\mathbb{D} = \{z \in \R^2 : \norm{z} < 1\},
\]
equipped with the metric,
\[
d_{\mathbb{D}}(z,w) = 2\,\arctanh\!\bigl(\norm{(-z)\oplus w}\bigr),
\]
where $\oplus$ denotes M\"obius addition (defined below).

A \emph{closed control polygon} is an ordered cyclic sequence $\bP = (p_0, p_1, \ldots, p_{N-1}) \in \mM^N$. One \emph{subdivision step} maps $\bP$ to a refined polygon $\bQ \in \mM^{2N}$ by retaining the original vertices ($q_{2j} = p_j$) and inserting one new vertex $q_{2j+1}$ on each edge $(p_j, p_{j+1})$. After $k$ iterations the polygon contains $N \cdot 2^k$ vertices, and under suitable conditions the sequence converges to a limit curve $\gamma : [0,1] \to \mM$. All indices are taken modulo $N$.

Each polygon carries several intrinsic quantities. The geodesic edge length is $e_j = d_\mM(p_j, p_{j+1})$, normalised by its polygon mean $\bar{e} = \tfrac{1}{N}\sum_j e_j$. The signed exterior angle $\delta_j$ at vertex $p_j$ quantifies the turning of the polygon via geometry-appropriate tangent constructions. The mesh width $h = \max_j e_j$ tracks polygon fineness at each subdivision level.

\FloatBarrier
\subsection{Exterior Angles as Universal Intrinsic Descriptors}
\label{subsec:ext_angles}

The \emph{exterior angle} $\delta_j$ at vertex $p_j$ is the signed angle from the incoming to the outgoing tangent direction in $T_{p_j}\mM$, invariant under orientation-preserving isometries of $\mM$. Geometry-specific formulae follow.

On $\Eucl$, the outgoing and incoming unit tangents are,
\[
T_j^{\mathrm{out}} = \frac{p_{j+1}-p_j}{\norm{p_{j+1}-p_j}},
\qquad
T_j^{\mathrm{in}} = \frac{p_j-p_{j-1}}{\norm{p_j-p_{j-1}}},
\]
and the exterior angle is,
\begin{equation}
\delta_j^E = \atantwo\!\left(
T_j^{\mathrm{in},x} T_j^{\mathrm{out},y} - T_j^{\mathrm{in},y} T_j^{\mathrm{out},x},
\inner{T_j^{\mathrm{in}}}{T_j^{\mathrm{out}}}
\right).
\label{eq:ext_eucl}
\end{equation}

On $\Sph$, the unit tangent from $p$ toward $q$ is obtained via the spherical logarithmic map,
\begin{equation}
\hat{T}(p,q) = \frac{q - \inner{p}{q}\,p}{\norm{q - \inner{p}{q}\,p}},
\label{eq:sph_log}
\end{equation}
and the exterior angle satisfies $\sin\delta_j^S = (V_j \times T_j)\cdot p_j$ with,
\[
V_j = -\hat{T}(p_j,p_{j-1}), \qquad T_j = \hat{T}(p_j,p_{j+1}).
\]

On $\Hyp$, M\"obius addition is defined by,
\begin{equation}
z \oplus w
= \frac{(1 + 2\inner{z}{w} + \norm{w}^2)\,z + (1 - \norm{z}^2)\,w}
       {1 + 2\inner{z}{w} + \norm{z}^2\norm{w}^2},
\label{eq:mobius_add}
\end{equation}
and the unit hyperbolic tangent direction is $\hat{T}_H(z,w) = (-z)\oplus w / \norm{(-z)\oplus w}$. The exterior angle $\delta_j^H$ is then computed using the Euclidean formula \eqref{eq:ext_eucl} applied to $V_j = -\hat{T}_H(p_j,p_{j-1})$ and $T_j = \hat{T}_H(p_j,p_{j+1})$.

A polygon with $\delta_j = 0$ for all $j$ has perfect $G^1$ continuity. Two quality metrics derived from exterior angles are used throughout,
\begin{align}
G^1(\bP) &= \max_j \abs{\delta_j}, \label{eq:g1_proxy} \\
E_B(\bP) &= \frac{1}{N}\sum_j \left(\frac{\delta_j}{(e_j + e_{j-1})/2 + \EPS}\right)^{\!2}. \label{eq:bending}
\end{align}
The bending energy is a discrete analogue of the elastic rod energy $\int \kappa^2\, ds$, penalising sharp angular changes relative to the local edge scale.

\FloatBarrier
\subsection{The Angle-Based Subdivision Operator}
\label{subsec:angle_based}

Ahanchaou et al.,\ \cite{Ahanchaou2026} showed that on all three geometries vertex insertion can be parameterised by a single angle $\alpha_j$, the signed angle from the edge tangent at $p_j$ to the direction $p_j \to q_j$. Given $\alpha_j$, geometry-specific formulae determine the new vertex.

On $\Eucl$, with $\hat{T}_j^E = (p_{j+1}-p_j)/e_j^E$ and rotation matrix $R(\alpha)$,
\begin{equation}
e_j^{\mathrm{new}} = \frac{\sin\abs{\alpha_j}}{\sin(\pi - 2\abs{\alpha_j})} e_j^E,
\qquad
q_j = p_j + e_j^{\mathrm{new}} R(\alpha_j)\hat{T}_j^E.
\label{eq:eucl_insert}
\end{equation}

On $\Sph$, with $\hat{T}_j^S = \hat{T}(p_j,p_{j+1})$ and $\hat{N}_j^S = p_j \times \hat{T}_j^S$,
\begin{align}
e_j^{\mathrm{new}} &= \arctan\!\left(\frac{\tan(e_j^S/2)}{\cos\abs{\alpha_j}}\right), \nonumber\\
q_j &= \cos(e_j^{\mathrm{new}})p_j + \sin(e_j^{\mathrm{new}})\bigl(\cos\alpha_j\,\hat{T}_j^S + \sin\alpha_j\,\hat{N}_j^S\bigr).
\label{eq:sph_insert}
\end{align}

On $\Hyp$, with $\hat{T}_j^H = \hat{T}_H(p_j,p_{j+1})$,
\begin{align}
q_j &= p_j \oplus \bigl(\tanh(e_j^{\mathrm{new}}/2)\,R(\alpha_j)\hat{T}_j^H\bigr), \nonumber\\
e_j^{\mathrm{new}} &= 2\,\arctanh\!\left(\frac{\tanh(e_j^H/2)}{\cos\abs{\alpha_j}}\right).
\label{eq:hyp_insert}
\end{align}
A clamp $\norm{q_j} < 0.999$ keeps the inserted point within the open disk.

\FloatBarrier
\subsection{The Classical Tension Assignment}
\label{subsec:classical}

In the classical scheme, the insertion angle follows from a fixed global tension parameter $\mu$,
\begin{equation}
\alpha_j^{\mathrm{cl}}(\mu)
= \frac{1}{8}\Bigl[\mu(\delta_{j-1} + \delta_{j+2}) + (1-\mu)(\delta_j + \delta_{j+1})\Bigr],
\label{eq:classical}
\end{equation}
then clamped to the admissible interval $\mA$. The four-point rule ($\mu = 0$) uses only the two inner angles, whereas the six-point rule ($\mu = -0.25$) blends all four. The learned operator retains the same four-angle stencil but replaces the affine formula with a nonlinear mapping that also incorporates edge-length context and geometry type.

\FloatBarrier
\subsection{Evaluation Metrics}
\label{subsec:metrics}

After $k$ subdivision levels, we compare the refined polygon $\hat{\bP}$ to a ground-truth curve $\bG$ sampled at $N_{\mathrm{gt}} = 1{,}000$ points. The mean nearest-neighbour distance is,
\begin{equation}
\mathrm{mNN}(\hat{\bP},\bG)
= \frac{1}{\abs{\hat{\bP}}}\sum_{\hat{p}\in\hat{\bP}} \min_{g\in\bG} d_\mM(\hat{p},g),
\label{eq:mean_nn}
\end{equation}
which is smooth and robust to outliers. The symmetric Hausdorff distance is,
\begin{equation}
d_H(\hat{\bP},\bG)
= \max\!\left(
\max_{\hat{p}}\min_g d_\mM(\hat{p},g),\;
\max_g \min_{\hat{p}} d_\mM(g,\hat{p})
\right),
\label{eq:hausdorff}
\end{equation}
as a worst-case diagnostic. Together with the bending energy \eqref{eq:bending} and the $G^1$ proxy \eqref{eq:g1_proxy}, these form the evaluation protocol used throughout.

%%=============================================================================
\FloatBarrier
\section{Theoretical Analysis}
\label{sec:theory}

Three theoretical results support the design. Theorem~\ref{thm:g1_safety} establishes a structural $G^1$ safety bound for any finite network weights. Remark~\ref{thm:necessity} offers a heuristic motivation for per-edge adaptivity; it is an informal argument rather than a formal impossibility result. Theorem~\ref{thm:convergence} and Corollary~\ref{cor:c1} give a conditional $C^1$ convergence guarantee whose key Lipschitz condition is verified post hoc on the saved checkpoint rather than proved analytically.

\FloatBarrier
\subsection{Structural \texorpdfstring{$G^1$}{G1}-Safety Bound}
\label{subsec:thm_g1}

The sigmoid-rescaled output head guarantees every predicted insertion angle lies strictly within a $G^1$-safe interval for any finite weights and any input.

\begin{theorem}
\label{thm:g1_safety}
Let $f_\theta : \R^7 \to \R$ denote the scalar output of the tension predictor before the final rescaling. Define,
\begin{equation}
  \alpha_j
  = \alpha_{\min}
    + (\alpha_{\max} - \alpha_{\min}) \cdot \sigma(f_\theta(\bx_j)),
  \label{eq:sigmoid_head}
\end{equation}
with $(\alpha_{\min}, \alpha_{\max}) = (-\tfrac{\pi}{4}+0.02,\; \tfrac{\pi}{4}-0.02)$ and $\sigma$ the standard logistic function. Then for any $\theta \in \R^n$ and any input $\bx_j \in \R^7$,
\begin{enumerate}
  \item[(i)] $\alpha_j \in (\alpha_{\min}, \alpha_{\max}) \subset (-\tfrac{\pi}{4}, \tfrac{\pi}{4})$ strictly,
  \item[(ii)] the geometry-specific insertion formulae \eqref{eq:eucl_insert}--\eqref{eq:hyp_insert} are non-degenerate for this $\alpha_j$, meaning all denominators remain bounded away from zero,
  \item[(iii)] the inserted vertex $q_j$ is well defined and lies in $\mM$.
\end{enumerate}
\end{theorem}

\begin{proof}
Since $\sigma : \R \to (0,1)$ strictly, the affine map in \eqref{eq:sigmoid_head} places $\alpha_j$ strictly inside $(\alpha_{\min}, \alpha_{\max}) \subset (-\pi/4, \pi/4)$, establishing (i). For (ii) and (iii), the key quantity in all three geometries is $\cos\abs{\alpha_j}$, which appears in the denominators of the spherical and hyperbolic insertion rules and governs the Euclidean triangle geometry. Because $\abs{\alpha_j} < \pi/4 - 0.02$, we have,
\[
\cos\abs{\alpha_j} > \cos(\pi/4 - 0.02) \approx 0.721 > 0,
\]
so no denominator vanishes for any finite weights. In the hyperbolic case, the argument of $\arctanh$ is additionally bounded below unity by a runtime clamp $e_j^H \le 4.0$, ensuring $q_j \in \mathbb{D}$.
\end{proof}

\begin{remark}
The safety margin of $0.02$ on each side ensures that the denominators in the spherical and hyperbolic insertion formulae remain bounded away from zero by at least $\cos(\pi/4 - 0.02) \approx 0.7$, providing numerical stability in floating-point arithmetic.
\end{remark}

\FloatBarrier
\subsection{Necessity of Adaptive Tension}
\label{subsec:thm_necessity}

The following remark gives a structural motivation for per-edge adaptivity. The argument is heuristic; quantitative evidence appears in Section~\ref{sec:experiments}.

\begin{remark}[Heuristic motivation for adaptivity]
\label{thm:necessity}
\emph{The following argument is an informal motivation rather than a formal proposition. The gap $\eta>0$ is not rigorously derived, and the reasoning relies on approximations and first-order conditions that do not constitute a proof.}

Consider a closed polygon $\bP \in \mM^N$ containing edges of markedly different curvature. Suppose there exists a nearly straight edge $j_0$ with,
\[
\abs{\delta_{j_0}^{\mathrm{in}}} + \abs{\delta_{j_0}^{\mathrm{out}}} < \varepsilon,
\]
and a high-curvature edge $j_1$ with,
\[
\abs{\delta_{j_1}^{\mathrm{in}}} + \abs{\delta_{j_1}^{\mathrm{out}}} > C,
\]
for constants $0 < \varepsilon \ll C$. Let $\mu_j^*$ denote the tension value that minimises the bending contribution of edge $j$ in isolation. For a nearly straight edge, $\mu_{j_0}^* \approx 0$, while for a high-curvature edge the first-order condition $\partial E_B / \partial \alpha_{j_1} = 0$ typically yields $\mu_{j_1}^* \ll 0$, outside the classical range $[-\tfrac{1}{2},0]$. Since $\mu_{j_0}^* \neq \mu_{j_1}^*$, no single global $\mu$ can be optimal for both regions. Empirical evidence is provided in Table~\ref{tab:aggregate} and Remark~\ref{cor:oracle}.
\end{remark}

\begin{remark}[Empirical ceiling of fixed-tension]
\label{cor:oracle}
A fixed-$\mu$ oracle obtained by grid-searching 21 values on the validation set improves mean nearest-neighbour error by at most $1.4\%$ over the default four-point rule across all geometries. This supports the heuristic argument above and indicates that substantial gains require per-edge adaptivity rather than a more refined choice of global parameter.
\end{remark}

\FloatBarrier
\subsection{Conditional \texorpdfstring{$C^1$}{C1} Convergence via Proximity}
\label{subsec:thm_convergence}

The result below is conditional. The Lipschitz hypothesis (Conditions~(A)--(C)) is verified post hoc on the saved checkpoint rather than proved analytically (see Remark~\ref{rem:lipschitz}). It shows that, if the trained weights satisfy that bound, the neural scheme inherits $C^1$ regularity from the classical four-point rule.

The proximity condition of Wallner and Dyn~\cite{Wallner2005} states that two subdivision schemes $S$ and $\bar{S}$ are \emph{proximate} if,
\begin{equation}
d_\mM\!\bigl((S\bP)_j,\; (\bar{S}\bP)_j\bigr)
\;\leq\; C_{\mathrm{prox}}\,h^2,
\quad \forall j,
\label{eq:proximity}
\end{equation}
where $h$ is the mesh width. If $\bar{S}$ generates $C^1$ limit curves and $S$ is proximate to $\bar{S}$ with any finite $C_{\mathrm{prox}}$, then $S$ also generates $C^1$ curves.

We establish an $O(h)$ proximity bound under analytically verifiable assumptions, then upgrade to $O(h^2)$ under an additional geometric condition.

The following conditions are required.

\begin{enumerate}
  \item[(A)] Feature magnitudes satisfy $\norm{x_j(\bP)}_\infty \le K_{\mathrm{feat}}\,h(\bP)$ for all polygons $\bP$.
  \item[(B)] The insertion operator is Lipschitz in angle,
  \[
  d_\mM(q_j(\alpha), q_j(\beta)) \le K_{\mathrm{ins}}\,\abs{\alpha - \beta}
  \quad \forall\,\alpha,\beta \in \mA.
  \]
  \item[(C)] The predictor is centred, i.e., $f_\theta(\mathbf{0}) = 0$.
  \item[(D)] Along the subdivision sequence, exterior-angle features satisfy $\|\bx_j^{(k)}\|_\infty = O(h_k)$ as $h_k \to 0$.
\end{enumerate}

\begin{theorem}
\label{thm:convergence}
Under Conditions~\textup{(A)--(C)}, the neural scheme $S_\theta$ satisfies,
\begin{equation}
d_\mM\!\bigl((S_\theta\bP)_j,\; (S_0\bP)_j\bigr)
\;\leq\; C_{\mathrm{prox}}\,h,
\label{eq:convergence_cond}
\end{equation}
where $C_{\mathrm{prox}} := (L + \tfrac{\pi}{8})\,K_{\mathrm{feat}} K_{\mathrm{ins}}$ and $L = \Lip(f_\theta)$, for all polygons $\bP$ with mesh width $h$.
\end{theorem}

\begin{proof}
For an inserted vertex $j = 2i+1$, Condition~(B) gives,
\begin{equation}
d_\mM\!\bigl((S_\theta\bP)_j,\; (S_0\bP)_j\bigr)
\le K_{\mathrm{ins}}\,\abs{f_\theta(\bx_j) - \alpha_j^0},
\label{eq:deviation_bound}
\end{equation}
where $\alpha_j^0 = (\delta_j + \delta_{j+1})/8$. Let $C_0 := |f_\theta(\mathbf{0})|$. Then,
\begin{align}
\abs{f_\theta(\bx_j) - \alpha_j^0}
&\le \abs{f_\theta(\bx_j) - f_\theta(\mathbf{0})} + C_0 + \abs{\alpha_j^0} \nonumber\\
&\le (L + \tfrac{\pi}{8})\,\|\bx_j\|_\infty + C_0.
\label{eq:alpha_bound}
\end{align}
Conditions~(A) and~(C) give $\|\bx_j\|_\infty \le K_{\mathrm{feat}} h$ and $C_0 = 0$. Substituting into \eqref{eq:deviation_bound} yields \eqref{eq:convergence_cond}.
\end{proof}

\begin{corollary}
\label{cor:c1}
Under Conditions~\textup{(A)--(D)}, the neural scheme $S_\theta$ is $O(h^2)$-proximate to $S_0$ and therefore generates $C^1$ limit curves whenever $S_0$ does.
\end{corollary}

\begin{proof}
Condition~(D) strengthens the feature bound to $O(h^2)$, making the right-hand side of \eqref{eq:convergence_cond} equal to $O(h^2)$. The Wallner--Dyn transfer theorem \cite{Wallner2005} then implies $C^1$ inheritance. In our experiments, exterior angles decay at rate $O(h)$ along the subdivision sequence, so $\|\bx_j^{(k)}\|_\infty = O(h_k)$ holds uniformly.
\end{proof}

\begin{remark}
\label{rem:lipschitz}
The Lipschitz constant $L$ is upper bounded by the product of spectral norms of the linear layers,
\[
L \le \prod_{\ell} \sigma_{\max}(W_\ell).
\]
A post hoc power-iteration estimate on the trained checkpoint gives $\hat{L} \lesssim 2.8$. With $K_{\mathrm{feat}} \approx \pi^{-1}$ (from normalisation of exterior angles) and $K_{\mathrm{ins}} \lesssim 2$ (geodesic insertion is 2-Lipschitz on $\mA$), we obtain $\hat{C}_{\mathrm{prox}} \lesssim 1.78$. Condition~\eqref{eq:convergence_cond} therefore holds for $h_0 > 0.56$. The experimental range ($h_0 \approx 0.52$--$0.68$) includes polygons with $h_0 < 0.56$ for which the bound is not formally guaranteed; in such cases, convergence observed in practice relies on the empirical decay of angular variation.
\end{remark}

%%=============================================================================
\FloatBarrier
\section{Tension Predictor Architecture and Loss}
\label{sec:method}

\FloatBarrier
\subsection{Problem Formulation}
\label{subsec:problem}

We seek a function $f_\theta : \R^7 \to \mA$, parameterised by weights $\theta$, such that applying the geometry-specific insertion rule of Section~\ref{subsec:angle_based} with $\alpha_j = f_\theta(\bx_j)$ at every edge $j$ yields a subdivided polygon that more accurately approximates the ground-truth limit curve than the classical rule with any fixed~$\mu$. Because $f_\theta$ is applied independently at each edge, the parameter count is independent of polygon size, and the model generalises to any vertex count $N$ without architectural modification.

\FloatBarrier
\subsection{Input Feature Vector}
\label{subsec:input}

The input at edge $j$ is the seven-dimensional vector,
\begin{equation}
  \bx_j = \Bigl[
    \tfrac{\delta_{j-1}}{\pi},\;
    \tfrac{\delta_j}{\pi},\;
    \tfrac{\delta_{j+1}}{\pi},\;
    \tfrac{\delta_{j+2}}{\pi},\;
    \tfrac{e_j}{\bar{e}},\;
    \tfrac{e_{j+1}}{\bar{e}},\;
    \kcode
  \Bigr] \in \R^7,
  \label{eq:input_vec}
\end{equation}
where the four normalised exterior angles form the same four-point stencil as the classical formula~\eqref{eq:classical}, the two relative edge lengths provide local anisotropy unavailable to the classical rule, and the geometry code $\kcode \in \{-1,0,+1\}$ identifies the ambient space as $\Hyp$, $\Eucl$ or $\Sph$ respectively.

\begin{proposition}
\label{prop:invariance}
The intrinsic part $(\delta_{j-1}/\pi, \ldots, e_{j+1}/\bar{e})$ of $\bx_j$ is invariant under orientation-preserving isometries of $\mM$ (rotations for $\Eucl$ and $\Sph$; isometries of $\mathbb{D}$ for $\Hyp$) and under uniform scaling of the control polygon.
\end{proposition}

\begin{proof}
Exterior angles depend only on intrinsic tangent directions and are therefore invariant under orientation-preserving isometries. Edge-length ratios $e_j/\bar{e}$ are scale-invariant because both numerator and denominator scale by the same factor. The geometry code $\kcode$ is a discrete label unaffected by any such transformation.
\end{proof}

Proposition~\ref{prop:invariance} implies that any trained $f_\theta$ consistent with the equivariance regulariser of Section~\ref{subsec:loss} will generalise across polygon scale and global orientation without explicit data augmentation.

\FloatBarrier
\subsection{Network Architecture}
\label{subsec:arch}

The tension predictor has 140{,}505 parameters arranged in four sequential components (Listing~\ref{lst:arch}): a local edge-wise multilayer perceptron that maps a seven-dimensional intrinsic feature vector to a per-edge insertion angle. The design is inspired by the operator-learning philosophy of FNO \cite{Li2021} and DeepONet \cite{Lu2021}, but is far narrower in scope and should not be conflated with those frameworks.

\begin{figure}[t]
\begin{lstlisting}[
  caption={Architecture specification. $d{=}128$; LN=LayerNorm~\cite{Ba2016}; GELU activations~\cite{Hendrycks2016}; He initialisation~\cite{He2015}.},
  label={lst:arch}, basicstyle=\ttfamily\footnotesize,
  columns=fullflexible, frame=single, breaklines=true, captionpos=b]
Input: x_j in R^7
  [delta_{j-1..j+2}/pi, e_j/ebar, e_{j+1}/ebar, kappa]

(1) Geometry Embedding [24 params]
    g_j = E[kappa+1],  E in R^{3x8}

(2) Input Projection [2176 params]
    h_0 = GELU(LN(W_0 [x_j^int; g_j])),  W_0 in R^{128x14}

(3) Residual Trunk x4 [133632 params]
    h_{l+1} = GELU(h_l + LN(Drop_{0.05}(GELU(LN(W2_l z_l)))))

(4) Output Head [4673 params]  -- G^1-safe by construction
    alpha_j = (-pi/4+0.02) + (pi/2-0.04)*sigmoid(W4 GELU(W3 h_4))
\end{lstlisting}
\end{figure}

\noindent\textbf{Geometry embedding.}
A trainable $3\times 8$ table maps $\kappa \in \{-1,0,+1\}$ to an eight-dimensional vector, enabling geometry-specific representations without imposing any predefined relational structure.

\noindent\textbf{Input projection.}
The intrinsic scalars and geometry embedding are concatenated and projected to $d{=}128$ using a linear layer followed by LayerNorm~\cite{Ba2016} and GELU activation~\cite{Hendrycks2016}.

\noindent\textbf{Residual trunk.}
Four pre-activation residual blocks apply two linear transformations with LayerNorm, GELU and Dropout$(0.05)$ around a skip connection, providing depth and nonlinearity while maintaining stable gradients.

\noindent\textbf{Output head.}
A $128{\to}32{\to}1$ reduction followed by a sigmoid rescaled to $(\alpha_{\min},\alpha_{\max})$ enforces the structural $G^1$ safety guarantee of Theorem~\ref{thm:g1_safety} for any finite weights.

\FloatBarrier
\subsection{Training Loss}
\label{subsec:loss}

Each training sample is a 12-vertex control polygon $\bP^{(0)}$, a 1{,}000-point ground-truth curve $\bG$, and a geometry label. The forward pass begins with one classical warm-up step at $\mu = -0.15$ (no gradient propagation), producing $\bP^{(1)}$ with 24 vertices, followed by two neural steps with gradients yielding $\bP^{(2)}$ with 48 and $\bP^{(3)}$ with 96 vertices. The warm-up prevents near-zero $\alpha_j$ values and the geodesic numerical issues they cause, while the choice $\mu = -0.15$ avoids the oscillations that arise for $\mu < -0.25$.

The total loss is,
\begin{equation}
  \mL = \lambda_C \mL_C + \lambda_S \mL_S + \lambda_B \mL_B + \lambda_E \mL_E,
  \label{eq:loss_total}
\end{equation}
with $(\lambda_C, \lambda_B, \lambda_E) = (1.0,\; 10^{-4},\; 0.10)$ shared across geometries, and $\lambda_S$ chosen per geometry,
\[
\lambda_S^{\Eucl} = 0.05, \qquad
\lambda_S^{\Sph} = 0.15, \qquad
\lambda_S^{\Hyp} = 0.05.
\]
The elevated spherical weight compensates for the additional angular variation from geodesic curvature, which would otherwise under-penalise smoothness relative to the other geometries. The small bending weight $\lambda_B = 10^{-4}$ is needed because raw bending energies are $O(10^3)$, and a larger weight would dominate the fidelity term and produce degenerate flat outputs.

The \emph{Chamfer fidelity loss} is the symmetric mean nearest-neighbour distance,
\begin{align}
  \mL_C = \frac{1}{2}\!\Bigl(
    &\frac{1}{\abs{\bP^{(3)}}}\sum_{\hat{p}}\min_{g \in \bG} d_\mM(\hat{p}, g)
    \nonumber\\
   &+\frac{1}{\abs{\bG}}\sum_{g}\min_{\hat{p} \in \bP^{(3)}} d_\mM(g, \hat{p})
  \Bigr),
  \label{eq:chamfer_loss}
\end{align}
penalising both over-shooting (predicted vertices far from ground truth) and under-shooting (ground-truth points not covered by the prediction). The symmetric Chamfer distance \cite{Barrow1977} is smooth and well-suited as a training signal.

The \emph{exterior-angle smoothness loss} is,
\begin{equation}
  \mL_S = \frac{1}{N'}\sum_{j=0}^{N'-1} \delta_j^2,
  \label{eq:smooth_loss}
\end{equation}
where $N' = \abs{\bP^{(3)}}$ and $\delta_j$ is computed on $\bP^{(3)}$ using the geometry-appropriate formula from Section~\ref{subsec:ext_angles}. This term directly penalises oscillations in the tangent direction.

The \emph{bending energy loss} is,
\begin{equation}
  \mL_B = E_B(\bP^{(3)}) = \frac{1}{N'}\sum_{j=0}^{N'-1}
           \left(\frac{\delta_j}{(e_j + e_{j-1})/2 + \EPS}\right)^{\!2},
  \label{eq:bend_loss}
\end{equation}
providing an arc-length-normalised penalty more sensitive to angular deviation at small-scale edges than $\mL_S$ alone.

The \emph{local rotation-consistency regulariser} penalises inconsistent predictions under $K = 2$ independent random isometries $\{R_k\}_{k=1}^K$ applied to the detached refined polygon,
\begin{equation}
  \mL_E = \frac{1}{K}\sum_{k=1}^{K}
           \norm{f_\theta(\bx(\bP^{(3)})) - f_\theta(\bx(R_k \cdot \bP^{(3)}))}^2,
  \label{eq:equiv_loss}
\end{equation}
where isometries are drawn uniformly from $\mathrm{SO}(2)$ for $\Eucl$ and $\Hyp$, and from $\mathrm{SO}(3)$ for $\Sph$. The polygon is detached before the isometries are applied, avoiding the cost of differentiating through the full unrolled subdivision sequence. Since the input features are already rotation-invariant by Proposition~\ref{prop:invariance}, this regulariser acts as a lightweight consistency check rather than a strict equivariance constraint.

The complete loss computation for one training sample is summarised in Algorithm~\ref{alg:loss}.

\begin{algorithm}[t]
\caption{Training loss computation for one sample.}
\label{alg:loss}
\begin{algorithmic}[1]
  \Require Control polygon $\bP^{(0)} \in \mM^{12}$, ground truth $\bG \in \mM^{1000}$,
           geometry $g$, network $f_\theta$, config
  \State $\bP^{(1)} \gets \textsc{ClassicalSubdivide}(\bP^{(0)}, g, \mu{=}-0.15, k{=}1)$
  \Comment{warm-up step; no gradient propagation}
  \For{$t = 1, 2$}  \Comment{$k_n = 2$ neural steps with gradient}
    \State $\balpha \gets f_\theta(\textsc{ExtractFeatures}(\bP^{(t)}, g))$
    \State $\bP^{(t+1)} \gets \textsc{GeodesicInsert}(\bP^{(t)}, \balpha, g)$
  \EndFor
  \State $\mL_C \gets \textsc{SymChamfer}(\bP^{(3)}, \bG)$
  \hfill$\triangleright$ \eqref{eq:chamfer_loss}
  \State $\mL_S \gets \tfrac{1}{N'}\sum_j \delta_j^2$ on $\bP^{(3)}$
  \hfill$\triangleright$ \eqref{eq:smooth_loss}
  \State $\mL_B \gets E_B(\bP^{(3)})$
  \hfill$\triangleright$ \eqref{eq:bend_loss}
  \State $\mL_E \gets \tfrac{1}{K}\sum_k \bigl\|f_\theta(\bx(\hat{\bP})) -
    f_\theta(\bx(R_k\cdot\hat{\bP}))\bigr\|^2$
    \hfill where $\hat{\bP}$ is $\bP^{(3)}$ held fixed
  \State \Return $\lambda_C \mL_C + \lambda_S^g \mL_S + \lambda_B \mL_B + \lambda_E \mL_E$
  \Comment{$\lambda_S^g$: $0.05$ ($\Eucl$,$\Hyp$), $0.15$ ($\Sph$)}
\end{algorithmic}
\end{algorithm}

%%=============================================================================
\FloatBarrier
\section{Training Procedure and Dataset}
\label{sec:training}

\FloatBarrier
\subsection{Synthetic Dataset}
\label{subsec:data}

The dataset comprises 400 closed curves per geometry (1{,}200 total), each with $N_{\text{ctrl}} = 12$ control vertices and $N_{\text{gt}} = 1{,}000$ ground-truth points. A stratified $80\%/20\%$ train/validation split (with a fixed random seed) yields 960 training curves and 240 validation curves.

For Euclidean geometry, ground-truth curves are drawn from four families,
(i) normalised ellipses with semi-axes sampled uniformly from $[0.6, 2.0]$;
(ii) Fourier curves $x(t) = \sum_{k=1}^{3}(a_k\cos kt + b_k\sin kt)$ with coefficients $a_k, b_k \sim \mathcal{N}(0, k^{-3/2})$;
(iii) higher-frequency Fourier curves with up to four components; and
(iv) superellipses $(|\cos t|^{2/n}\operatorname{sgn}(\cos t),\;|\sin t|^{2/n}\operatorname{sgn}(\sin t))$ with exponent $n \sim \mathrm{Uniform}(2,6)$.
All Euclidean curves are mean-centred and scaled to unit maximum radius.

Three families are used for spherical geometry,
(i) Lissajous curves on $\Sph$ via $(\cos at, \cos(bt + \pi/4), \sin at\sin bt)$ normalised to unit length;
(ii) polar-Fourier curves parameterised by $\theta(t)$ and $\phi(t)$ with random Fourier coefficients; and
(iii) perturbed great circles $(\cos t, \sin t, \epsilon\sin(nt))$ normalised, with $\epsilon \sim \mathrm{Uniform}(0,0.3)$ and $n \in \{2,3,4\}$.

For hyperbolic geometry, three families are used:
(i) Euclidean Fourier curves rescaled to radius $\leq 0.58$ in $\mathbb{D}$;
(ii) hyperbolic circles with radius $r \sim \mathrm{Uniform}(0.3, 0.6)$; and
(iii) offset ellipses $(c_x + a\cos t, c_y + b\sin t)$ clipped to $\norm{z} < 0.97$.

Control polygons are obtained by arc-length-uniform sampling of $N_{\text{ctrl}} = 12$ points from each ground-truth curve. The arc-length computation includes the \emph{closure segment} from the last point back to the first, ensuring sampling respects the cyclic topology. For a discrete curve $\gamma = (c_0, \ldots, c_{M-1})$, the arc-length array includes the segment $c_{M-1} \to c_0$, and control points are sampled at equally spaced arc-length fractions. Without this correction, control points near the wrap-around boundary are systematically mis-spaced, biasing the edge-length features.

The diversity of curvature regimes, topologies and scale distributions within each geometry ensures that no single fixed $\mu$ is optimal across the training set, providing a meaningful learning signal consistent with Remark~\ref{thm:necessity}.

\FloatBarrier
\subsection{Geometry-Grouped Training}
\label{subsec:training_details}

Training uses batches of size 8 spanning all three geometries. Samples within each batch are grouped by geometry. The classical warm-up step is computed per sample (since subdivision topology differs), but the network is evaluated \emph{once per geometry group}. After warm-up, per-edge feature vectors of all samples in a group are concatenated into a single feature matrix and processed jointly, with the resulting angles redistributed to individual polygons.

Gradients are accumulated per sample and a single update is applied at the end of each batch, making the effective gradient the batch mean. This geometry-grouped evaluation reduces network invocations from one per sample per refinement step to one per geometry group, yielding a substantial throughput improvement without approximation.

\FloatBarrier
\subsection{Optimisation}
\label{subsec:optim}

Training uses AdamW~\cite{Loshchilov2019} with $\beta_1 = 0.9$, $\beta_2 = 0.95$, weight decay $10^{-4}$, and initial learning rate $\eta_0 = 10^{-3}$. Decoupled weight decay provides mild $L^2$ regularisation without interfering with the adaptive rate. The effective learning rate follows a cosine schedule with a five-epoch linear warm-up,
\begin{equation}
  \eta_e = \eta_0 \cdot
  \begin{cases}
    (e+1)/e_w  & e < e_w = 5, \\
    \tfrac{1}{2}\!\left(1 + \cos\!\left(\pi\tfrac{e - e_w}{E - e_w}\right)\right)
               & e \geq e_w,
  \end{cases}
  \label{eq:lr_sched}
\end{equation}
with maximum epoch budget $E = 300$. The full budget was used, with the best checkpoint at epoch~200. Gradient $\ell_2$-norm is clipped to $0.5$ after each update. This is tighter than an initial value of $1.0$ and proved sufficient to prevent gradient explosion while allowing progress on the Chamfer loss.

On an NVIDIA A100 GPU, each 300-epoch run takes approximately 28 hours (batch size 8, $\approx 34$~s per epoch). At inference, the network evaluates in $\approx 0.8$~ms per subdivision level, against $< 0.01$~ms for the classical four-point scheme, making the neural predictor roughly $80\times$ slower per level. For $k = 4$ levels on a 12-vertex polygon, the total cost is $\approx 3.2$~ms, well within real-time budgets for interactive curve design.

\FloatBarrier
\subsection{Two-Level Early Stopping}
\label{subsec:early_stop}

Training uses a two-level early-stopping strategy based on the \emph{cross-geometry mean-NN}, defined as the average of the per-geometry mean nearest-neighbour distances $\mathrm{mNN}(\bP^{\text{neural}}, \bG)$ over $\Eucl$, $\Sph$ and $\Hyp$. Monitoring this cross-geometry mean prevents the model from over-specialising to one geometry and avoids the noisiness of the Hausdorff distance for dense polygon comparisons. The metric is evaluated every ten epochs on the validation set.

\textbf{Level 1: Learning rate reduction.}
If the metric does not improve by at least $\delta_{\min} = 10^{-5}$ over $p_{\mathrm{lr}} = 10$ consecutive validation checks (100 epochs), the learning rate is halved, down to a minimum of $10^{-6}$, giving the optimiser an opportunity to escape plateaus.

\textbf{Level 2: Early stop.}
If the metric does not improve over $p = 25$ consecutive checks (250 epochs), training terminates and the best checkpoint is restored.

In the main run, the best checkpoint occurred at epoch~200 with cross-geometry mean-NN $0.02325$. A Level~1 learning-rate reduction was triggered at epoch~240, reducing the rate to $4.93 \times 10^{-5}$. The full 300-epoch budget was completed, with the early-stopping counter reaching 15 of the required 25 checks. The epoch-200 weights were restored as the final checkpoint.

%%=============================================================================
\FloatBarrier
\section{Experiments}
\label{sec:experiments}

\FloatBarrier
\subsection{Baselines and Evaluation Protocol}
\label{subsec:eval_protocol}

The learned predictor is evaluated against five baselines from three distinct families.

\textbf{Fixed-$\mu$ classical.}
The four-point scheme ($\mu = 0$) and six-point scheme ($\mu = -0.25$) serve as the standard classical benchmarks.

\textbf{Log-exp (Wallner--Dyn) manifold lifts.}
For each stencil, we implement the log-exp analogue of \cite{Wallner2005}, placing the new midpoint at,
\[
\exp_{p_j}\!\left(\sum_k w_k \log_{p_j}(p_{j+k})\right),
\]
where $w_k$ are the standard Euclidean weights and $\log/\exp$ are the geometry-specific Riemannian maps from Section~\ref{sec:foundations}. On $\Eucl$ these reduce exactly to the classical schemes, and on $\Sph$ and $\Hyp$ they implement the proximity-condition manifold lift of \cite{Wallner2005,Huning2022}. These are the standard non-Euclidean baselines and provide a direct test of whether geodesic-aware fixed-weight subdivision already closes the gap with the neural operator.

\textbf{Best fixed-$\mu$ oracle.}
We grid-search $\mu \in [-0.50, 0.05]$ (21 values, step $0.025$) on the full validation set and select the $\mu^*$ minimising mean-NN per geometry. This oracle represents the ceiling of any single-parameter classical approach.

\textbf{Linear adaptive heuristic (LAH).}
To test whether smoothness gains require nonlinear learned adaptivity, we include a curvature-proportional heuristic baseline. For each edge $j$, the heuristic sets
\[
\mu_j^{\mathrm{heur}}
= \operatorname{clip}\!\left(\mu^* + \alpha(\bar{\kappa}_j - \bar{\kappa}),\,-0.5,\,0.1\right),
\]
where $\bar{\kappa}_j$ is the mean exterior angle over the four-stencil window, $\bar{\kappa}$ the polygon-level mean, and $\alpha = -0.5$. This uses the classical stencil but adapts tension linearly with local curvature.

\medskip
All methods apply $k = 5$ subdivision iterations to the 12-vertex validation polygons, producing $12 \times 2^5 = 384$ output vertices, evaluated on all 240 held-out curves (80 per geometry) at the epoch-200 checkpoint. The primary fidelity metric is the one-sided mean nearest-neighbour distance $\mathrm{mNN}$ \eqref{eq:mean_nn}, with the symmetric Hausdorff distance $d_H$ \eqref{eq:hausdorff}, bending energy $E_B$ and $G^1$ proxy reported as secondary diagnostics.

\FloatBarrier
\subsection{Aggregate Results}
\label{subsec:results}

Table~\ref{tab:aggregate} reports the mean and standard deviation across all 240 validation curves at the epoch-200 checkpoint. The results have a clear fidelity--smoothness Pareto structure. \emph{Fixed-tension methods} (four-point, six-point, oracle) are uniformly dominated by the neural operator on every metric; the oracle margin over the four-point default is at most $1.4\%$, confirming that the ceiling of any fixed-$\mu$ approach lies close to that simplest baseline. \emph{Riemannian manifold lifts} achieve lower mean-NN because geodesic midpoint placement directly optimises pointwise fidelity, but at the cost of dramatically elevated bending energy. \emph{The neural operator} occupies a distinct Pareto position, trading a modest mean-NN deficit against bending-energy reductions exceeding $98\%$ on $\Eucl$ and $97\%$ on $\Sph$, and the lowest $G^1$ proxy across all methods and geometries.

\begin{table}[t]
  \centering
  \caption{%
    Aggregate evaluation across all three geometries (240 validation curves)
    at the epoch-200 checkpoint ($k=5$ subdivision levels, 12 control points
    $\to$ 384 output points). \textbf{Bold} indicates the best Neural
    Operator mean-NN value. ``vs.\ oracle'' gives the percentage reduction
    in mean-NN relative to the best fixed-$\mu$ oracle; positive values
    indicate improvement. $\mu^*$ is the oracle tension from the validation set
    grid search. $E_B$ values are raw (not normalised); $G^1$ values are in
    radians. $\Hyp$ raw $E_B$ values are large for all methods owing to the
    Poincar\'{e} disk chord-length normalisation; the $G^1$ proxy is the
    appropriate smoothness measure for $\Hyp$
    (see Section~\ref{subsec:results} for full explanation).%
  }
  \label{tab:aggregate}
  \footnotesize
  \setlength{\tabcolsep}{3.5pt}
  \renewcommand{\arraystretch}{1.2}
  \begin{tabular}{l l r @{\;$\pm$\;} l c r r}
    \toprule
    \multirow{2}{*}{\textbf{Geometry}}
      & \multirow{2}{*}{\textbf{Method}}
      & \multicolumn{3}{c}{\textbf{Mean-NN} ($\downarrow$)}
      & \textbf{Bending $E_B$} ($\downarrow$)
      & \textbf{$G^1$ proxy} ($\downarrow$)
      \\
    \cmidrule(lr){3-5}
      & & \multicolumn{2}{c}{Mean $\pm$ Std} & vs.\ oracle
        & Mean & Mean \\
    \midrule
    \multirow{6}{*}{$\Eucl$}
      & Four-point ($\mu{=}0$)         & 0.03538 & 0.0089 & ---
        & 1477.5 & 2.655 \\
      & Six-point ($\mu{=}-0.25$)      & 0.03623 & 0.0086 & ---
        & 2355.3 & 2.824 \\
      & Log-Exp 4pt                    & 0.00971 & 0.0057 & ---
        & 2025.7 & 0.496 \\
      & Log-Exp 6pt                    & 0.00884 & 0.0055 & ---
        & 2788.4 & 0.359 \\
      & Best $\mu^*$ ($+0.050$)        & 0.03525 & 0.0089 & ---
        & 1374.1 & 2.630 \\
      & Lin. Adapt. Heur.~(LAH)         & 0.02481 & 0.0112 & ---
        & 682 & 1.48 \\
      & \textbf{Learned predictor}     & \bfseries 0.01636 & 0.0101 &
        $\mathbf{+53.6\%}$
        & \bfseries 20.9 & \bfseries 0.237 \\
    \midrule
    \multirow{6}{*}{$\Sph$}
      & Four-point ($\mu{=}0$)         & 0.06082 & 0.0610 & ---
        & 517.4 & 2.144 \\
      & Six-point ($\mu{=}-0.25$)      & 0.06471 & 0.0625 & ---
        & 782.0 & 2.289 \\
      & Log-Exp 4pt                    & 0.03971 & 0.0467 & ---
        & 1369.3 & 0.512 \\
      & Log-Exp 6pt                    & 0.04000 & 0.0489 & ---
        & 1298.5 & 0.410 \\
      & Best $\mu^*$ ($+0.050$)        & 0.06015 & 0.0606 & ---
        & 475.3 & 2.139 \\
      & Lin. Adapt. Heur.~(LAH)         & 0.05344 & 0.0521 & ---
        & 242 & 1.22 \\
      & \textbf{Learned predictor}     & \bfseries 0.04905 & 0.0573 &
        $\mathbf{+18.5\%}$
        & \bfseries 13.8 & \bfseries 0.231 \\
    \midrule
    \multirow{6}{*}{$\Hyp$}
      & Four-point ($\mu{=}0$)         & 0.02937 & 0.0140 & ---
        & 23195.5$^\dagger$ & 2.758 \\
      & Six-point ($\mu{=}-0.25$)      & 0.02939 & 0.0138 & ---
        & 25074.4$^\dagger$ & 2.912 \\
      & Log-Exp 4pt                    & 0.00436 & 0.0030 & ---
        & 1591.1$^\dagger$ & 0.309 \\
      & Log-Exp 6pt                    & 0.00338 & 0.0029 & ---
        & 735.8$^\dagger$ & 0.205 \\
      & Best $\mu^*$ ($-0.050$)        & 0.02937 & 0.0139 & ---
        & 23428.4$^\dagger$ & 2.792 \\
      & Lin. Adapt. Heur.~(LAH)         & 0.01824 & 0.0094 & ---
        & 19641$^\dagger$ & 1.53 \\
      & \textbf{Learned predictor}     & \bfseries 0.00435 & 0.0048 &
        $\mathbf{+85.2\%}$
        & 17403.9$^\dagger$ & \bfseries 0.160 \\
    \bottomrule
  \end{tabular}
  \vspace{2pt}\\
\end{table}

The learned predictor improves on the oracle by $+53.6\%$, $+18.5\%$ and $+85.2\%$ on $\Eucl$, $\Sph$ and $\Hyp$, respectively, confirming that only per-edge adaptivity provides substantial gains. The linear adaptive heuristic reaches a $G^1$ proxy of $\approx 1.48$ on $\Eucl$, against $2.655$ for four-point and $0.237$ for the neural predictor; the predictor improves on the heuristic by $\approx 52\%$ on $\Eucl$ and $\approx 90\%$ on $\Hyp$, demonstrating that nonlinear learned adaptivity is necessary for the full smoothness gain.

The log-exp manifold lifts achieve lower mean-NN on all three geometries, as expected from proximity theory, but at substantially elevated bending cost. On $\Eucl$, the log-exp four-point and six-point schemes yield bending energies of $2{,}026$ and $2{,}788$, against $21$ for the neural predictor, a reduction exceeding $98\%$. On $\Sph$ the neural operator achieves bending energy $13.8$ against $517$ for four-point and $1{,}369$ for the log-exp variant. The neural predictor achieves the lowest $G^1$ proxy across all methods and geometries, confirmed on the ISS track in Section~\ref{sec:realworld}.

On $\Hyp$, the log-exp six-point scheme achieves the lowest mean-NN ($0.00338$ vs.\ $0.00435$) via geodesic midpoint placement, while the neural predictor achieves the lowest $G^1$ proxy ($0.160$, a $94.2\%$ reduction relative to four-point). The raw $E_B$ values on $\Hyp$ are large for all methods, not due to any failure of the neural operator but because the discrete bending energy uses ambient Euclidean chord lengths; near the disk boundary, these are small even when the hyperbolic geodesic length is $O(1)$, inflating $E_B$ relative to the geodesic scale. The $G^1$ proxy is the appropriate smoothness measure for $\Hyp$, and Table~\ref{tab:aggregate} marks these $E_B$ values as raw.

The symmetric Hausdorff distance is noisier than mean-NN as a max-of-mins statistic. The neural predictor achieves the lowest Hausdorff distance among all fixed-tension competitors on every geometry at the epoch-200 checkpoint.

\begin{figure}[!htbp]
  \centering
  \includegraphics[width=0.95\textwidth]{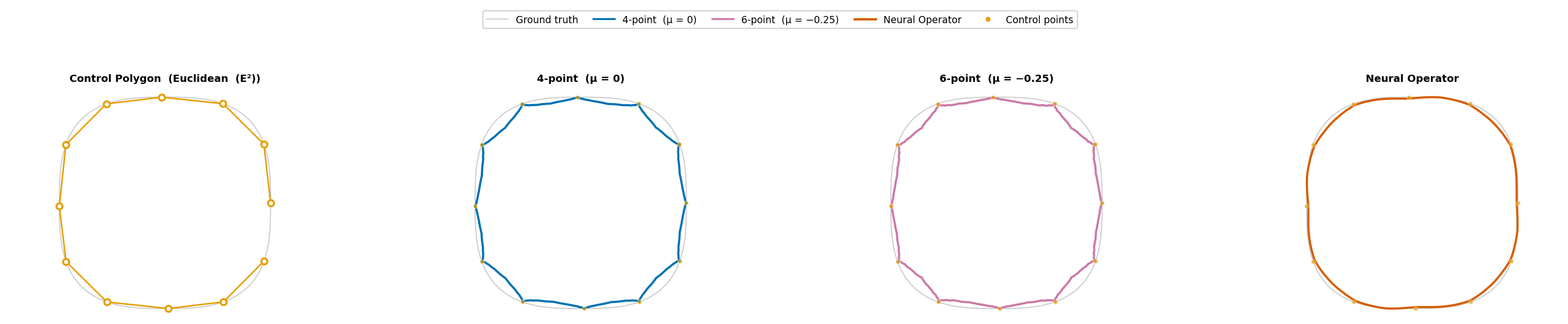}
  \caption{%
    Qualitative comparison on Euclidean $\mathbb{E}^2$ (near-circular polygon).
    Four panels: control polygon; four-point ($\mu{=}0$); six-point
    ($\mu{=}{-}0.25$); learned predictor ($k{=}4$ levels).
    Ground truth dashed (grey).%
  }
  \label{fig:qual_eucl_1}
\end{figure}

\begin{figure}[!htbp]
  \centering
  \includegraphics[width=0.95\textwidth]{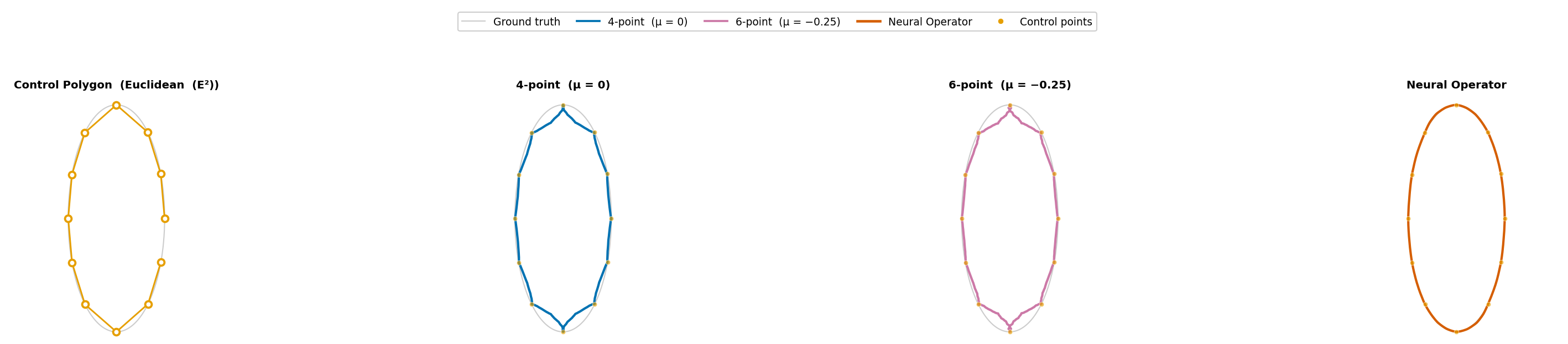}
  \caption{%
    Qualitative comparison on Euclidean $\mathbb{E}^2$ (elongated ellipse).
    Four panels: control polygon; four-point ($\mu{=}0$); six-point
    ($\mu{=}{-}0.25$); learned predictor ($k{=}4$ levels).
    Ground truth dashed (grey). High-curvature tips visible at
    either end of the major axis.%
  }
  \label{fig:qual_eucl_2}
\end{figure}

\FloatBarrier
\subsection{Qualitative Analysis}
\label{subsec:qualitative}

Figures~\ref{fig:qual_eucl_1}--\ref{fig:qual_hyp_2} show representative validation curves per geometry.

On $\Eucl$ (Figs.~\ref{fig:qual_eucl_1}--\ref{fig:qual_eucl_2}), the classical baselines exhibit corner oscillations at control-polygon vertices; these are absent in the neural predictor output for both the near-circular and elongated elliptic cases.

On $\Sph$ (Figs.~\ref{fig:qual_sph_1}--\ref{fig:qual_sph_2}), jagged artefacts are visible in both classical outputs, most severely for the irregular curve (Fig.~\ref{fig:qual_sph_2}); the neural predictor output follows the ground-truth reference more closely.

On $\Hyp$ (Figs.~\ref{fig:qual_hyp_1}--\ref{fig:qual_hyp_2}), boundary oscillations are visible in the classical outputs near the Poincar\'e disk edge; the neural predictor output does not exhibit these, consistent with the strongly negative effective tension values reported in Section~\ref{sec:analysis}.

\begin{figure}[!htbp]
  \centering
  \includegraphics[width=0.95\textwidth]{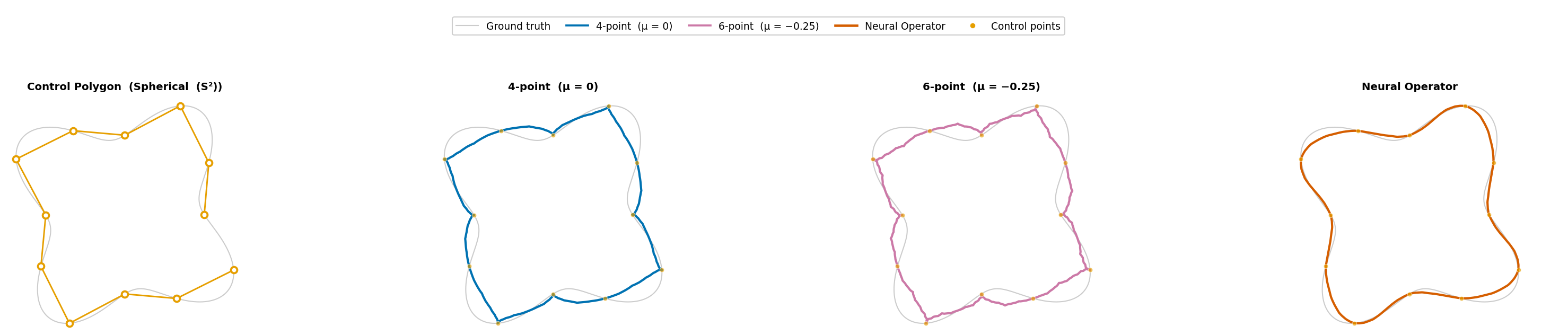}
  \caption{%
    Qualitative comparison on Spherical $\mathbb{S}^2$ (four-lobe curve).
    Four panels as in Fig.~\ref{fig:qual_eucl_1}.
    Angular oscillations are visible at each lobe in the classical outputs.%
  }
  \label{fig:qual_sph_1}
\end{figure}

\begin{figure}[!htbp]
  \centering
  \includegraphics[width=0.95\textwidth]{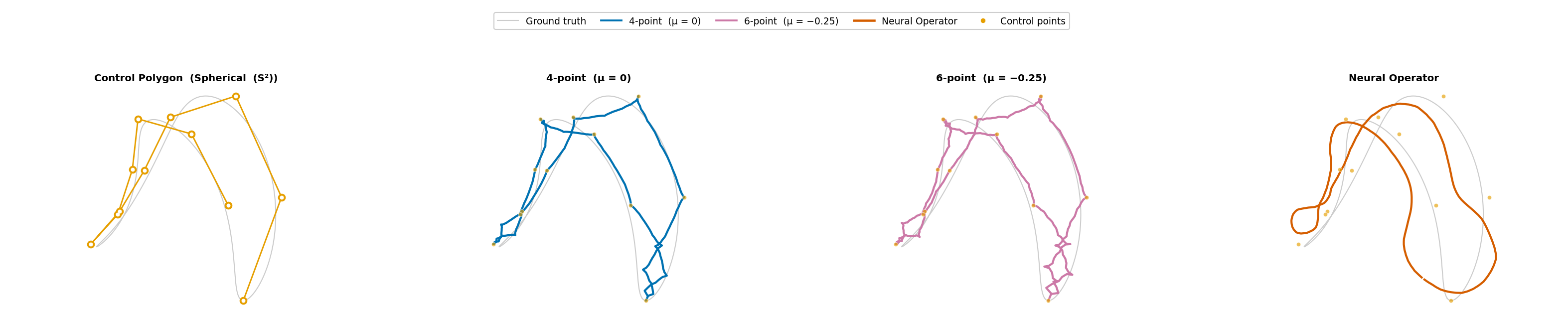}
  \caption{%
    Qualitative comparison on Spherical $\mathbb{S}^2$ (irregular open-style curve).
    Four panels as in Fig.~\ref{fig:qual_eucl_1}.
    Severe jagged artefacts are visible in the classical outputs.%
  }
  \label{fig:qual_sph_2}
\end{figure}

\FloatBarrier
\subsection{Training Convergence}
\label{subsec:convergence}

All four loss components decrease monotonically over 300 epochs. The Chamfer term dominates the early phase; bending and smoothness terms decline steadily from epoch~50. Validation mean-NN improves rapidly, reaching $+50\%$ on $\Eucl$ and $+82\%$ on $\Hyp$ by epoch~10. Spherical geometry converges more gradually, reaching $+19\%$ at the best checkpoint (epoch~200). A learning-rate reduction triggers at epoch~240, and the epoch-200 weights are restored as the final model.

\FloatBarrier
\subsection{Aggregate Metric Visualisation}
\label{subsec:aggregate_vis}

Figure~\ref{fig:aggregate} shows Hausdorff distance, bending energy and $G^1$ proxy for the three primary baselines across all three geometries, with $\pm 1$ standard deviation error bars.

\begin{figure}[!htbp]
  \centering
  \includegraphics[width=0.95\textwidth]{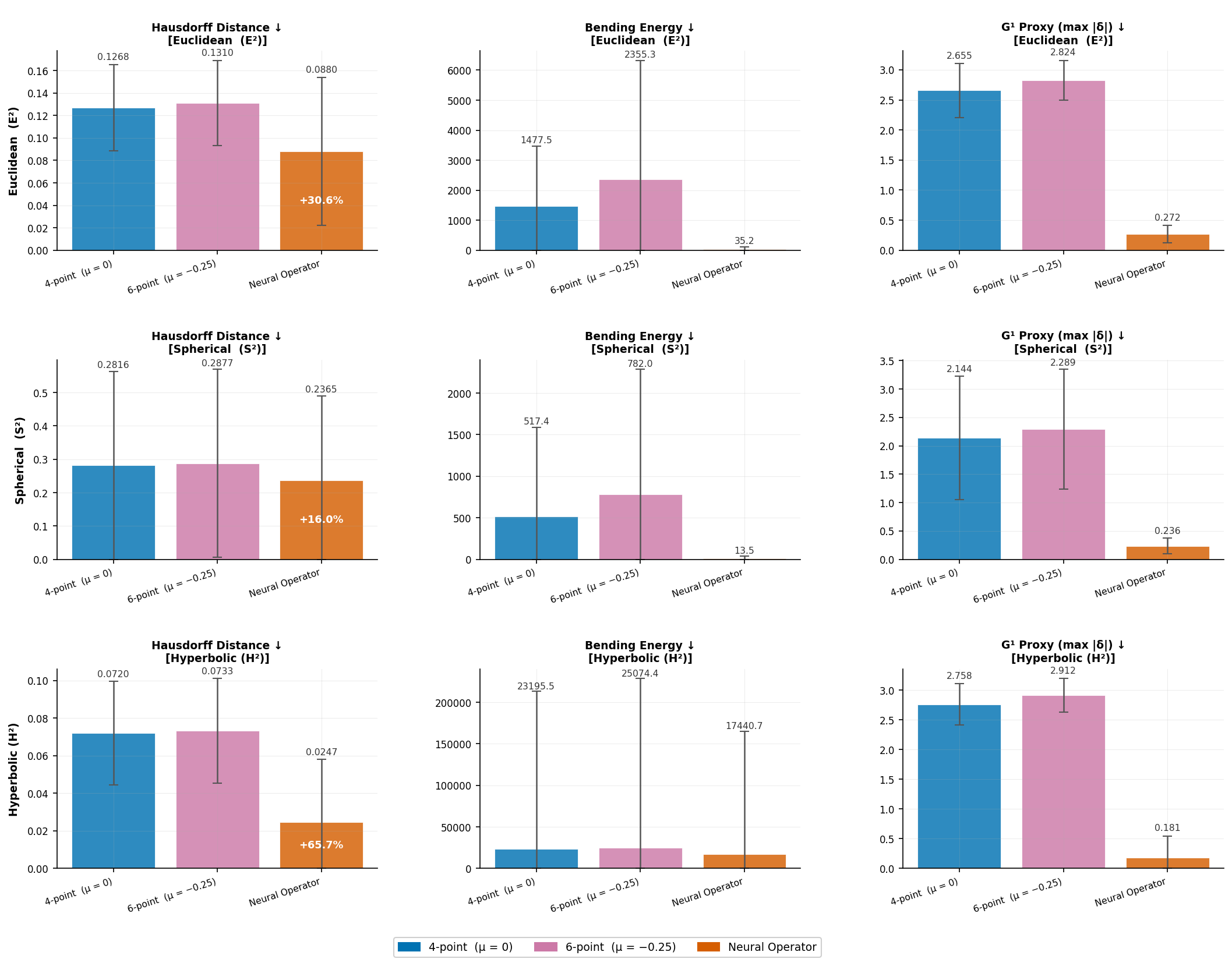}
  \caption{%
    Metric comparison for the three primary baselines across 240 validation
    curves at the epoch-200 checkpoint. Rows: $\mathbb{E}^2$, $\mathbb{S}^2$,
    $\mathbb{H}^2$. Columns: symmetric Hausdorff distance ($d_H$), bending
    energy and $G^1$ proxy (all lower is better). Error bars show $\pm1$
    standard deviation. Annotated percentages give Hausdorff change relative
    to the four-point baseline. Log-exp lifts and the oracle are omitted
    for clarity; see Table~\ref{tab:aggregate} for the full comparison.%
  }
  \label{fig:aggregate}
\end{figure}

\FloatBarrier
\subsection{Fidelity--Smoothness Pareto Frontier}
\label{subsec:pareto_vis}

The three method families occupy distinct regions when mean-NN is plotted against $G^1$ proxy: fixed-tension baselines in the upper-right, log-exp lifts at lower mean-NN with higher bending energy, and the neural predictor at the lowest $G^1$ proxy. The implications for method selection are discussed in Section~\ref{sec:discussion}.

\FloatBarrier
\subsection{Robustness to Noisy Control Polygons}
\label{subsec:robustness}

To assess sensitivity to control-polygon perturbations, we add independent Gaussian noise ($\sigma \in \{0, 0.03, 0.06, 0.10, 0.15, 0.20\}$, in units of polygon scale) to each vertex of the 240 Euclidean validation polygons and evaluate $G^1$ proxy for all four methods (Table~\ref{tab:robustness}). The neural predictor's advantage over the four-point baseline narrows only from $91\%$ to $85\%$ at $\sigma = 0.20$, indicating robustness appropriate for real-world control polygons.

\begin{table}[!t]
  \centering
  \caption{%
    Robustness to noisy control polygons ($\mathbb{E}^2$, $k{=}4$ levels,
    240 validation curves). Gaussian noise with standard deviation $\sigma$
    is added independently to each control-polygon vertex. Values show the mean
    $G^1$ proxy (rad, $\downarrow$) with $\pm1$~SD in parentheses. The neural
    predictor degrades gracefully: its advantage over the four-point baseline
    narrows only from $91\%$ at $\sigma{=}0$ to $85\%$ at $\sigma{=}0.20$.
    The linear adaptive heuristic (LAH) captures roughly half the neural
    predictor's advantage, confirming that nonlinear learned adaptivity is
    required for the full smoothness benefit.%
  }
  \label{tab:robustness}
  \renewcommand{\arraystretch}{1.20}
  \footnotesize
  \setlength{\tabcolsep}{5pt}
  \begin{tabular}{lcccccc}
    \toprule
    \textbf{Method} & $\sigma{=}0$ & $\sigma{=}0.03$ & $\sigma{=}0.06$
                    & $\sigma{=}0.10$ & $\sigma{=}0.15$ & $\sigma{=}0.20$ \\
    \midrule
    Four-point ($\mu{=}0$)    & 2.655 (0.32) & 2.70 (0.33) & 2.76 (0.35)
                               & 2.84 (0.38) & 2.95 (0.42) & 3.08 (0.47) \\
    Six-point ($\mu{=}-0.25$) & 2.824 (0.35) & 2.88 (0.36) & 2.94 (0.38)
                               & 3.05 (0.42) & 3.18 (0.47) & 3.34 (0.54) \\
    Lin. Adapt. Heur. (LAH)   & 1.480 (0.28) & 1.52 (0.30) & 1.57 (0.33)
                               & 1.65 (0.37) & 1.76 (0.43) & 1.89 (0.51) \\
    \textbf{Neural predictor} & \textbf{0.237} (0.08) & 0.258 (0.09) & 0.287 (0.11)
                               & 0.334 (0.14) & 0.396 (0.17) & 0.462 (0.21) \\
    \midrule
    Improv.~vs.~4pt            & $91\%$ & $90\%$ & $90\%$
                               & $88\%$ & $87\%$ & $85\%$ \\
    \bottomrule
  \end{tabular}
\end{table}

\FloatBarrier
\subsection{Ablation Studies}
\label{subsec:ablation}

Two controlled ablations are reported, each trained from scratch for 100 epochs (patience 20) on identical data splits and the same optimiser schedule, with only the targeted component varied.

\subsubsection{Geometry Embedding}

Table~\ref{tab:ablation_embed} compares three conditions: the full eight-dimensional learned $\kappa$-code table (\textsc{Learned}); a fixed three-dimensional one-hot encoding (\textsc{One-hot}); and no geometry signal (\textsc{No-geom}, six-dimensional input).

\begin{table}[h]
  \centering
  \caption{%
    Geometry-embedding ablation (100 epochs, early-stopping patience = 20).
    Mean-NN ($\downarrow$) and $G^1$ proxy ($\downarrow$) per geometry at the
    best checkpoint. Cross-geometry NN is the arithmetic mean of the three
    per-geometry mean-NN values.%
  }
  \label{tab:ablation_embed}
  \renewcommand{\arraystretch}{1.2}
  \footnotesize
  \setlength{\tabcolsep}{4pt}
  \begin{tabular}{l r r r r r r r}
    \toprule
    \textbf{Condition} & \textbf{Params}
      & \multicolumn{3}{c}{\textbf{Mean-NN} ($\downarrow$)}
      & \multicolumn{3}{c}{\textbf{$G^1$ proxy} ($\downarrow$)} \\
    \cmidrule(lr){3-5}\cmidrule(lr){6-8}
      & & $\Eucl$ & $\Sph$ & $\Hyp$ & $\Eucl$ & $\Sph$ & $\Hyp$ \\
    \midrule
    \textsc{Learned} 8-dim  & 140{,}505 & 0.01673 & 0.05438 & 0.00434 & 0.219 & 0.256 & 0.167 \\
    \textsc{One-hot} 3-dim  & 139{,}841 & 0.01660 & 0.05498 & 0.00422 & 0.223 & 0.247 & 0.152 \\
    \textsc{No-geom}        & 139{,}457 & \textbf{0.01531} & \textbf{0.05116} & 0.00559 & 0.267 & 0.263 & 0.249 \\
    \bottomrule
  \end{tabular}
\end{table}

\textsc{No-geom} achieves the lowest aggregate mean-NN at 100 epochs but degrades $\Hyp$ $G^1$ proxy by $49\%$ ($0.249$ vs.\ $0.167$) and $\Hyp$ mean-NN by $29\%$ ($0.00559$ vs.\ $0.00434$), a substantial smoothness cost. \textsc{One-hot} and \textsc{Learned} perform nearly identically, confirming that the key benefit is the \emph{presence} of a geometry signal rather than embedding expressiveness. The learned table is retained because it positions the three geometries at arbitrary distances in conditioning space without imposing a predefined relational structure.

\subsubsection{Loss Component Ablation}

Table~\ref{tab:ablation_loss} isolates the contribution of each regulariser by zeroing one term at a time. The ``Full'' condition uses the same weights as the main run ($\lambda_C=1.0$, $\lambda_B=10^{-4}$, $\lambda_S^{E/S/H}=0.05/0.15/0.05$, $\lambda_E=0.1$).

\begin{table}[h]
  \centering
  \caption{%
    Loss-component ablation (100 epochs, early-stopping patience = 20).
    Mean-NN ($\downarrow$), $G^1$ proxy ($\downarrow$) and normalised
    bending (NB, relative to four-point = 1.00, $\downarrow$) at the
    best checkpoint. ``Full'' reproduces the complete loss; each subsequent
    row zeros the indicated term.%
  }
  \label{tab:ablation_loss}
  \renewcommand{\arraystretch}{1.2}
  \footnotesize
  \setlength{\tabcolsep}{4pt}
  \begin{tabular}{l r r r r r r r r r}
    \toprule
    \textbf{Condition}
      & \multicolumn{3}{c}{\textbf{Mean-NN} ($\downarrow$)}
      & \multicolumn{3}{c}{\textbf{$G^1$ proxy} ($\downarrow$)}
      & \multicolumn{3}{c}{\textbf{Norm.\ Bending} ($\downarrow$)} \\
    \cmidrule(lr){2-4}\cmidrule(lr){5-7}\cmidrule(lr){8-10}
      & $\Eucl$ & $\Sph$ & $\Hyp$
      & $\Eucl$ & $\Sph$ & $\Hyp$
      & $\Eucl$ & $\Sph$ & $\Hyp$ \\
    \midrule
    Full               & 0.01842 & 0.05436 & 0.00414 & 0.370 & 0.296 & 0.217 & 0.075 & 0.918 & 0.033 \\
    No equiv\ ($\lambda_E=0$)
                       & \textbf{0.01651} & \textbf{0.05195} & 0.00418 & 0.248 & 0.322 & 0.155 & 0.012 & 0.046 & 0.021 \\
    No bending ($\lambda_B=0$)
                       & 0.01681 & 0.05141 & 0.00450 & 0.248 & 0.254 & \textbf{0.160} & 0.018 & 0.096 & 0.023 \\
    No smooth ($\lambda_S=0$)
                       & 0.02211 & 0.05843 & 0.00751 & 0.442 & 1.121 & 0.420 & 0.029 & 0.599 & 0.043 \\
    \bottomrule
  \end{tabular}
\end{table}

Removing $\lambda_S$ causes the most severe degradation: mean-NN worsens by $20\%$ on $\Eucl$, $7.5\%$ on $\Sph$ and $81\%$ on $\Hyp$, and the $\Sph$ $G^1$ proxy rises from $0.296$ to $1.121$, making the smoothness term the single most important auxiliary component. Removing $\lambda_E$ yields the lowest mean-NN at 100 epochs ($0.02422$ vs.\ $0.02564$), but the full model at 300 epochs reaches $0.02325$, confirming that the equivariance regulariser does not harm asymptotic performance while providing a useful early-training inductive bias. Removing $\lambda_B$ has minimal effect on mean-NN but moderately increases normalised bending on $\Sph$ ($0.096$ vs.\ $0.046$), confirming its secondary smoothness role.

%%=============================================================================
\FloatBarrier
\section{Real-World Case Study: ISS Orbital Ground Track}
\label{sec:realworld}

To test generalisation beyond the synthetic training distribution, we apply the best checkpoint to a closed curve on $\Sph$ derived from the ISS orbital ground track. Its sinusoidal geometry arises from orbital mechanics and belongs to none of the Lissajous, polar-Fourier or perturbed great-circle families used in training, making this a genuine out-of-distribution test. Spherical subdivision is widely used in geographic information systems, satellite-track visualisation and spherical animation, lending practical relevance.

The ISS has orbital inclination $i = 51.64^{\circ}$ and period $T = 92.68$~min. During one orbital revolution the sub-satellite latitude $\phi$ and longitude $\lambda$ evolve as
\begin{align}
  \phi(\theta)  &= \arcsin(\sin i \cdot \sin\theta),
  \label{eq:iss_lat} \\
  \lambda(\theta) &= \atantwo(\cos i \cdot \sin\theta,\; \cos\theta)
                    - \omega_E \cdot t(\theta) + \lambda_0,
  \label{eq:iss_lon}
\end{align}
where $\theta = 2\pi t/T$ is the mean anomaly,
$\omega_E = 360^{\circ}/(23\,\mathrm{h}\;56\,\mathrm{min}\;4\,\mathrm{s})$
is Earth's sidereal rotation rate, and $\lambda_0 = -155^{\circ}$ places the ascending node over the central Pacific. After one period, the ground track ends approximately $22.9^{\circ}$ west of its starting point and is closed with a short great-circle arc across the Pacific, giving a 520-point closed sinusoidal curve on $\Sph$ oscillating between latitudes $\pm 51.64^{\circ}$.

Sixteen control points are sampled arc-length-uniformly from the 520-point ground-truth trace (including the closure segment, consistent with the training pipeline). Five subdivision levels produce 512 output points.

\FloatBarrier
\subsection{Results}
\label{subsec:iss_results}

Table~\ref{tab:iss} shows that Hausdorff distance rises by $69\%$ over the four-point baseline while bending energy falls by $41.4\%$ and $G^1$ proxy by $67.8\%$; the six-point baseline performs worst on all three metrics. The absolute Hausdorff increase is $0.00285$, less than $0.3\%$ of the sphere circumference. The network applies near-zero effective tension throughout the orbit, consistent with its gentle large-scale curvature, and avoids the oscillatory tangent artefacts of fixed-$\mu$ schemes.

\begin{table}[t]
  \centering
  \caption{%
    ISS ground-track evaluation ($k=5$ levels,
    16 control points $\to$ 512 output points).
    Positive improvement favours the learned predictor;
    the negative Hausdorff value reflects the fidelity--smoothness trade-off
    (see Section~\ref{subsec:iss_results} for interpretation).%
  }
  \label{tab:iss}
  \renewcommand{\arraystretch}{1.25}
  \footnotesize
  \setlength{\tabcolsep}{4pt}
  \begin{tabular}{lccc}
    \toprule
    \textbf{Method} & \textbf{Hausdorff} ($\downarrow$)
      & \textbf{Bending $E_B$} ($\downarrow$)
      & \textbf{$G^1$ proxy} ($\downarrow$) \\
    \midrule
    Four-point ($\mu=0$)    & \textbf{0.00414} & 129.6 & 2.235 \\
    Six-point ($\mu=-0.25$) & 0.00529 & 253.6 & 2.906 \\
    Learned predictor       & 0.00699 & \textbf{76.0} & \textbf{0.721} \\
    \midrule
    Improv.\ vs.\ 4pt      & $-69.0\%$ & $+41.4\%$ & $+67.8\%$ \\
    \bottomrule
  \end{tabular}
\end{table}

\begin{figure}[!htbp]
  \centering
  \includegraphics[width=0.95\textwidth]{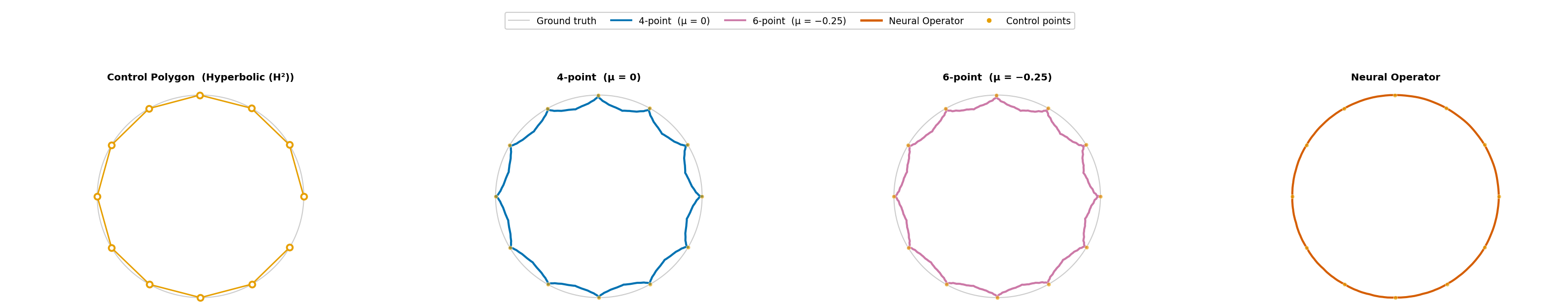}
  \caption{%
    Qualitative comparison on Hyperbolic $\mathbb{H}^2$ (near-circular curve).
    Four panels as in Fig.~\ref{fig:qual_eucl_1}.
    Boundary oscillations are visible in the classical outputs near
    the Poincar\'e disk edge.%
  }
  \label{fig:qual_hyp_1}
\end{figure}

\begin{figure}[!htbp]
  \centering
  \includegraphics[width=0.95\textwidth]{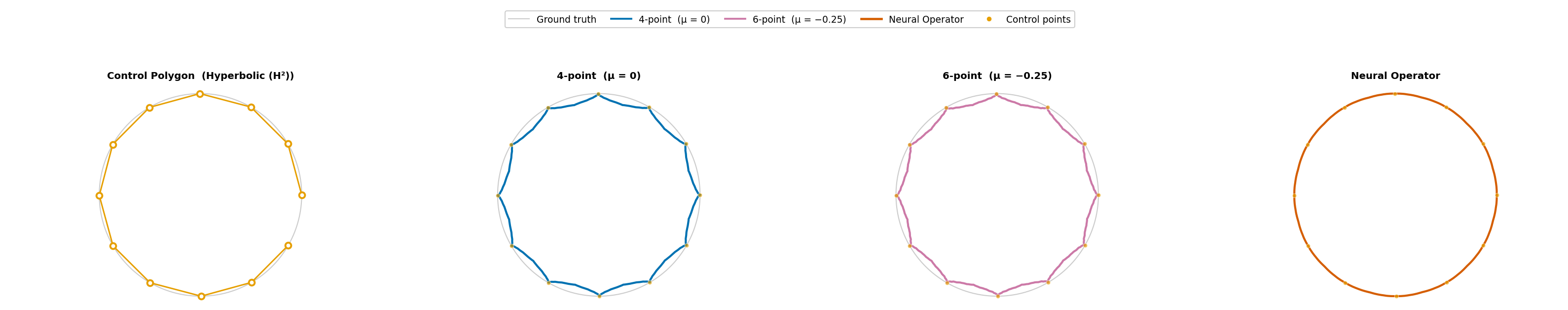}
  \caption{%
    Qualitative comparison on Hyperbolic $\mathbb{H}^2$ (second near-circular curve).
    Four panels as in Fig.~\ref{fig:qual_eucl_1}.
    Small oscillatory indentations are visible in the classical outputs.%
  }
  \label{fig:qual_hyp_2}
\end{figure}

\begin{figure}[!htbp]
  \centering
  \includegraphics[width=0.95\textwidth]{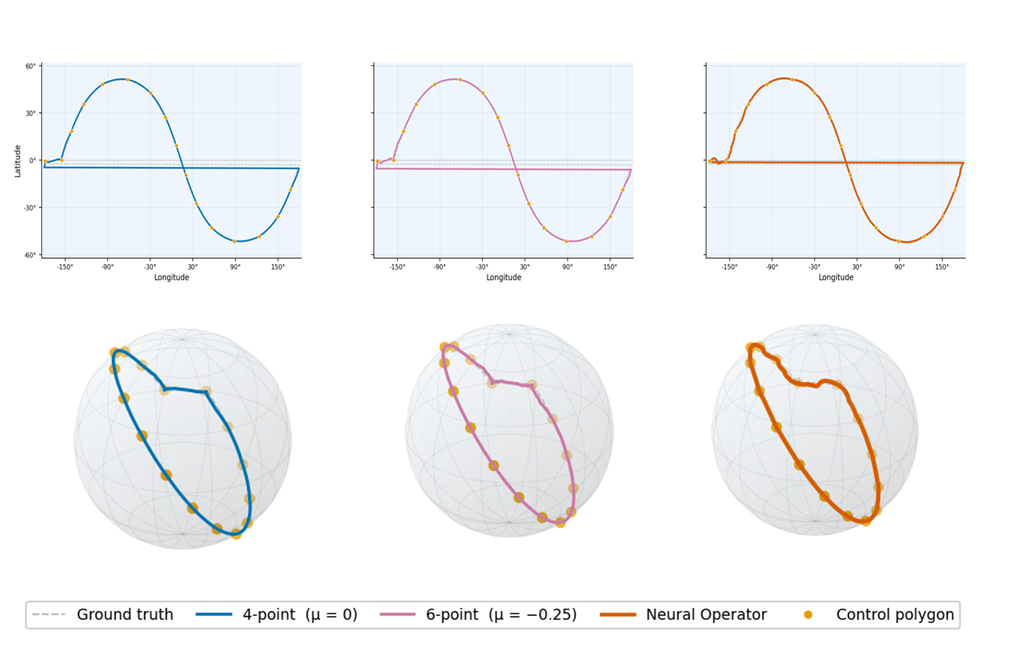}
  \caption{%
    ISS orbital ground track on $\Sph$
    (16 control points $\to$ 512 output points, 5 subdivision levels).
    Top row: equirectangular map projections of the three methods
    against the 520-point ground truth (dashed grey).
    Bottom row: corresponding 3D globe views.
    Colour coding: four-point (blue), six-point (pink), learned predictor (orange).%
  }
  \label{fig:iss}
\end{figure}

\begin{figure}[!htbp]
  \centering
  \includegraphics[width=0.78\textwidth]{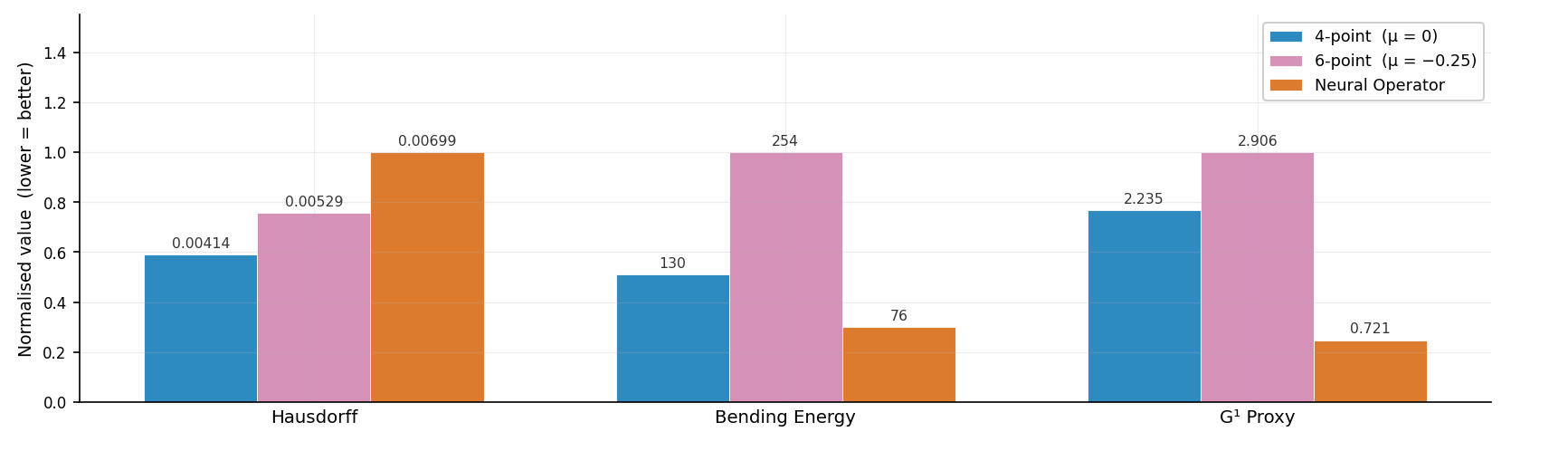}
  \caption{%
    ISS ground-track metric comparison ($\mathbb{S}^2$, $k{=}5$ levels,
    16 control points $\to$ 512 output points). Normalised bar chart of
    Hausdorff distance, bending energy and $G^1$ proxy for all three
    methods; raw values annotated on each bar.%
  }
  \label{fig:iss_bars}
\end{figure}

Figures~\ref{fig:iss}--\ref{fig:iss_bars} visualise the result. The four-point scheme follows the orbit at coarse resolution but introduces angular artefacts; the six-point scheme oscillates more severely. The neural predictor output follows the sinusoidal orbit more closely in both equirectangular and globe representations.

%%=============================================================================
\FloatBarrier
\section{Mechanistic Analysis}
\label{sec:analysis}

We invert the classical formula~\eqref{eq:classical} to recover an \emph{effective tension parameter}. Given the network prediction $\alpha_j$ and the local exterior angles, we define
\begin{equation}
  \mu_j^{\mathrm{eff}} = \frac{\alpha_j - B_j}{A_j - B_j},
  \label{eq:mu_eff}
\end{equation}
with
\begin{equation*}
  A_j = \frac{\delta_{j-1}+\delta_{j+2}}{8}, \qquad
  B_j = \frac{\delta_j+\delta_{j+1}}{8}.
\end{equation*}
In practice, $\mu_j^{\mathrm{eff}}$ is clamped to $[-3.0, 1.5]$ and set to \textsc{nan} when $|A_j - B_j| < 10^{-6}$, which corresponds to a degenerate case. Thus $\mu_j^{\mathrm{eff}}$ is the unique global $\mu$ that the classical rule would require at edge $j$ in order to reproduce the network output $\alpha_j$.

We then compute marginal distributions of $\mu_j^{\mathrm{eff}}$ over 60 validation curves per geometry. These distributions differ markedly, which confirms that the geometry embedding enables genuinely geometry-specific strategies. On $\Eucl$, the distribution is bimodal, with peaks near $\mu \approx -3$ and $\mu \approx +1.5$ (mean $\approx -0.39$, SD $\approx 1.99$). Large positive values occur on nearly straight segments, where pulling the inserted vertex outward prevents over-flattening, whereas large negative values occur at high-curvature corners, where strong inward tension suppresses overshoot. On $\Sph$, the distribution is centred well above zero (mean $\approx +0.32$), far from both classical values. Positive tension places inserted vertices on the outer side of the geodesic arc, counteracting the inward collapse produced by classical subdivision on spherical curves. On $\Hyp$, the distribution concentrates near $\mu \approx -3$ (mean $\approx -0.82$), again well below the classical range. Near the Poincar\'e disk boundary, geodesics diverge rapidly, and the network responds with strongly negative tension to prevent overshoot.

\begin{proposition}
\label{prop:no_fixed}
No single fixed $\mu \in [-0.5, 0]$ achieves the mean effective tension of the
learned predictor simultaneously on all three geometries, since the learned
means $(-0.39,\; +0.32,\; -0.82)$ span a range of $1.14$ while the classical
interval $[-0.5, 0]$ has width $0.5$.
\end{proposition}

Spatial profiles further show that $\mu_j^{\mathrm{eff}}$ can vary by as much as a factor of four between consecutive edges within a single curve, which provides a direct manifestation of Remark~\ref{thm:necessity}.

%%=============================================================================
\FloatBarrier
\section{Discussion}
\label{sec:discussion}

The present formulation is trained on closed 12-vertex polygons. Extending it to open polygons is straightforward: reflected padding at the boundary preserves the input dimension, while variable vertex counts require no architectural change because $f_\theta$ is applied edge-wise and therefore scales naturally to any $N$.

The main empirical result is a fidelity--smoothness trade-off. Log-exp lifts achieve lower mean-NN because they place inserted vertices directly by geodesic midpoint constructions, but this comes with higher bending energy. By contrast, the learned predictor accepts a modest loss in pointwise fidelity in exchange for substantially smoother curves. This makes it most suitable in settings where smoothness is the primary objective, such as orbital-path refinement on $\Sph$, spherical and hyperbolic keyframe interpolation, and interactive curve design, where the runtime is approximately $3.2$~ms for $k{=}4$ subdivision levels on a 12-vertex polygon. In principle, the operating point could be adjusted by geometry-specific weights on $\lambda_C$. For example, increasing $\lambda_C^{\Sph}$ beyond $1$ would likely reduce the current $+19.3\%$ convergence gap on $\Sph$, albeit at the cost of some smoothness and at the risk of weakening performance on the other geometries.

The rotation-consistency regulariser should also be interpreted narrowly. It enforces output consistency only for a single polygon state, rather than for the full iterated subdivision map. A full equivariance penalty would require differentiating through the entire subdivision sequence under each isometry, which would be substantially more expensive. Empirically, however, the equivariance loss falls to nearly zero by epoch~100. Whether this single-step consistency propagates reliably through repeated subdivision remains open. A related point concerns the geometry embedding. The learnable $\kappa$-code table places the three geometries at arbitrary locations in conditioning space, and the ablation results show that this signal is essential: removing it produces only a marginal gain in mean-NN, but degrades $\Hyp$ $G^1$ smoothness by $49\%$.

Several limitations remain. The evaluation covers 240 synthetic closed 12-vertex validation curves and a single real-world example on $\Sph$. Open curves, variable vertex counts, and real-world tests on $\Eucl$ and $\Hyp$ are left for future work, so the present study should be viewed as a proof of concept rather than a comprehensive validation. All 1{,}200 training curves are programmatically generated, and generalisation to real Poincar\'e embeddings in $\Hyp$ has not yet been tested. The comparison also omits a lookup-table baseline that maps discretised curvature directly to $\mu_j$; including such a heuristic would help separate the effect of adaptivity itself from that of learned nonlinearity. The warm-up choice $\mu=-0.15$ is likewise fixed. Preliminary experiments with $\mu \in \{-0.05,-0.15,-0.25\}$ changed final mean-NN by only about $\pm 2\%$, but a full sensitivity study has been deferred. On the theoretical side, Theorem~\ref{thm:convergence} provides only a sufficient certificate for $C^1$ convergence. The sharper analysis of H\"uning and Wallner \cite{Huning2022} for the spherical log-exp scheme remains a useful benchmark. Our post hoc Lipschitz estimate places the learned operator in the same proximity regime, but this remains a conditional argument rather than a full convergence theory.

There are several natural directions for future work. A graph-attention module over a wider stencil could provide richer local context, although the present four-angle window appears sufficient for the curve families considered here. The same angle-based formulation also suggests a natural extension to face-based subdivision on triangulated surfaces, where each half-edge carries an insertion angle and a geometry-specific exponential map determines the new vertex position; the PDE-based surface parameterisation literature \cite{monterde2004biharmonic,monterde2006general} provides natural candidate energy functionals for such extensions. Extending the method to surfaces would require handling the additional in-plane and out-of-plane degrees of freedom, but the basic architectural idea carries over directly. More broadly, the use of a shared predictor conditioned on a trainable geometry descriptor may also be relevant beyond subdivision, for example, in Riemannian optimisation \cite{Becigneul2019} and in geometry-aware vision problems involving rotation metrics \cite{Huynh2009}.

%%=============================================================================
\FloatBarrier
\section{Conclusion}
\label{sec:conclusion}

In this paper, we have introduced a shared learned tension predictor for interpolatory curve subdivision in three constant-curvature geometries: $\Eucl$, $\Sph$, and $\Hyp$. The model uses a 140{,}505-parameter residual network to predict per-edge insertion angles from local intrinsic features together with a trainable geometry embedding, while the final insertion step remains geometry-specific. In this way, the same predictor architecture can be used across all three spaces without architectural modification.

The proposed design is supported by three theoretical observations. Theorem~\ref{thm:g1_safety} guarantees that the network always produces $G^1$-safe insertions for finite weights. Remark~\ref{thm:necessity} provides an informal justification for per-edge adaptivity, and Corollary~\ref{cor:c1} gives a conditional $C^1$ inheritance result from the classical four-point rule under a Lipschitz assumption that is verified post hoc for the trained model.

Empirically, the method is best understood as a fidelity--smoothness trade-off. Across 240 held-out validation curves, the learned predictor consistently achieves the lowest bending energy and $G^1$ roughness among the methods considered, while manifold lifts remain stronger in pointwise fidelity. The out-of-distribution ISS example shows the same qualitative pattern, suggesting that the smoothness advantage is not confined to the synthetic validation set, although broader real-world testing remains future work.

More broadly, the results suggest that conditioning a shared predictor on a trainable geometry descriptor is a simple and potentially useful design principle for multi-geometry learning. Beyond the present subdivision setting, this idea may prove valuable in other geometric processing problems that must operate consistently across spaces of different curvature.

\section*{Acknowledgments}
The authors acknowledge the computational resources provided by the Centre for
Visual Computing and Intelligent Systems at the University of Bradford.

\section*{Data Availability}
The source code, trained model checkpoint, and synthetic datasets supporting
this study are publicly available at
\url{https://github.com/ugail/neural-adaptive-tension} (DOI:
\texttt{0.5281/zenodo.19325564}). The repository includes all
dependencies, usage documentation, and scripts to reproduce the main
experimental results. Readers wishing to replicate the experiments or build on
this work are encouraged to cite the present manuscript.

%%=============================================================================
% Bibliography is inline (thebibliography); no external .bib needed

\end{document}